\documentclass{article}
\usepackage{iclr2027_conference,times}

\usepackage[utf8]{inputenc}
\usepackage[T1]{fontenc}
\usepackage{hyperref}
\usepackage{url}
\usepackage{booktabs}
\usepackage{amsthm}
\newtheorem{lemma}{Lemma}
\usepackage{amsfonts}
\usepackage{amsmath}
\usepackage{amssymb}
\usepackage{nicefrac}
\usepackage[expansion=false]{microtype}
\usepackage{enumitem}
\usepackage{graphicx}
\usepackage{xcolor}
\newif\ifrevmarks
\revmarksfalse
\definecolor{revcolor}{HTML}{0050E0}
\DeclareRobustCommand{\rev}[1]{\ifrevmarks{\color{revcolor}#1}\else#1\fi}
\hypersetup{colorlinks=true, linkcolor=blue!45!black, citecolor=blue!45!black, urlcolor=blue!45!black}

\newcommand{\Wbase}{W_{\text{base}}}
\newcommand{\Wsafe}{W_{\text{safe}}}

\newcommand{\Lcap}{\mathcal{L}_{\text{cap}}}
\newcommand{\Lsafe}{\mathcal{L}_{\text{safe}}}
\newcommand{\Fcap}{F_{\text{cap}}}
\newcommand{\qnorm}{{q_{\text{norm}}}}
\DeclareMathOperator{\tr}{tr}

\ifrevmarks
\title{The Geometry of Refusal:\\
Why Post-Hoc Safety Is Fragile\\ and Pretraining-Time Safety Persists}
\else
\title{The Geometry of Refusal:\\
Why Post-Hoc Refusal Is Fragile\\ and Pretraining-Time Refusal Persists}
\fi

\author{Srikanth Malla\textsuperscript{*}, Chiho Choi \& Joon Hee Choi \\ Samsung Semiconductor US \\ \texttt{\{srikanth.m, chiho1.choi, jh4.choi\}@samsung.com}}

\iclrfinalcopy
\begin{document}
\maketitle
\lhead{}
{\renewcommand{\thefootnote}{}\footnotetext{\textsuperscript{*}Correspondence to \texttt{srikanth.m@samsung.com}.}}

\begin{abstract}
Post-hoc safety training (RLHF, DPO) is the dominant way to align large language models, yet jailbreaks~\citep{zou2023gcg}, fine-tuning attacks~\citep{qi2024finetuning}, and activation-space edits~\citep{arditi2024refusal} keep recovering the behaviors it was meant to remove. We give this fragility one geometric explanation and follow it into pretraining. We measure the safety update $\Delta = \Wsafe - \Wbase$ against the curvature of the model's capabilities (the empirical Fisher of a capability loss). Across five model families, post-hoc safety lands in a \emph{suppression} regime: $\Delta$ is nearly orthogonal to the capability directions, and its small in-subspace part concentrates on a few high-curvature ones. The update is \emph{thin} but \emph{sharp}, a refusal gate laid over intact capabilities rather than erasure of them. A kernel-immobility lemma explains why such an update can only mask a capability, not remove it, \rev{and a little benign fine-tuning switches the gate off}: $100$ benign examples cut the AdvBench refusal of Qwen-2.5-7B-Instruct and Llama-3-8B-Instruct by $35$ to $38$~pp.

Following the account into pretraining, a pretraining-checkpoint sweep of OLMo-2-1B~\citep{olmo2_2025} shows the features that refusal attaches to emerging in a sharp transition between $1$B and $63$B pretraining tokens. We then use the account constructively: models trained from scratch with safety co-training spread \emph{continuously} across pretraining reach $87$ to $98\%$ AdvBench refusal that the same attack erodes by only $2$ to $14$ pp at every scale from $410$M to $6.9$B, at a small cost on short-answer capability probes; a \emph{windowed} schedule of equal total safety weight installs no refusal. \rev{Persistence, not timing, decides whether refusal survives. It does not keep output safer: by a harmful-output classifier, the attack raises harmful completions by $8$ to $12$ pp after post-hoc safety and by $11$ to $27$ pp after co-training.}
\end{abstract}

\begin{figure}[t]
\centering
\includegraphics[width=\textwidth]{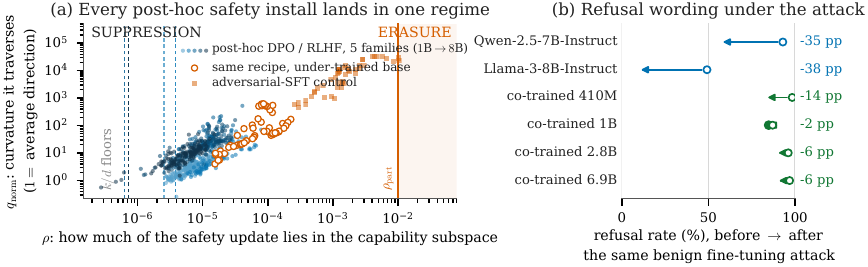}
\caption{\textbf{Post-hoc safety lands orders of magnitude short of erasure and loses $35$ to $38$ pp of refusal to a benign fine-tuning attack; co-training loses $2$ to $14$ pp\rev{ of refusal wording}.} All points measured. \textbf{(a)} Where safety updates land in the $(\rho, \qnorm)$ plane, one point per transformer block and capability probe. Post-hoc DPO/RLHF across five model families (blue, light to dark by size) is \emph{thin but sharp}: two to three orders of magnitude below the erasure boundary $\rho_\text{part}$, yet at its peak blocks above each family's own random floor $k/d$ (dashed; the five span $6.2\times$). Erasure is reachable but post-hoc safety never reaches it: an adversarial-SFT control built to saturate the capability subspace (squares) crosses $\rho_\text{part}$. Applying post-hoc DPO to an under-trained base (circles) lands elsewhere, which Sec.~\ref{sec:m3} traces to a developmental transition. This panel is post-hoc installs only: a co-trained model never received a patch, so no such $\Delta$ exists (App.~\ref{sec:app-limits}); its claim is panel (b). \textbf{(b)} Why that geometry matters: AdvBench refusal rate before $\rightarrow$ after the \emph{same} 100-example benign-SFT attack, regex-scored throughout so rows are comparable. Post-hoc installs lose $35$ to $38$ pp (Llama-3 to $11\%$); co-trained refusal loses $2$ to $14$ pp, ending at $84$ to $91\%$ at every scale to $6.9$B. Regex tracks the refusal \emph{gate}, which is what the geometry predicts; App.~\ref{sec:app-detector} repeats both attack tables under a harmful-output classifier\rev{, which reverses this comparison: the attack raises harmful completions by $8$ to $12$ pp after post-hoc safety and by $11$ to $27$ pp after co-training}.}
\label{fig:overview}
\end{figure}

\section{Introduction}
\label{sec:intro}

Large language models are aligned \emph{after} pretraining, not during it. The standard pipeline (pretrain, then RLHF or DPO on a preference dataset) assumes that any undesired behavior a base model can exhibit can later be removed by a targeted fine-tuning signal. A consistent body of recent empirical work undermines that assumption: a few hundred steps of fine-tuning on innocuous data undo RLHF-installed refusals~\citep{qi2024finetuning}, targeted optimization finds suffix strings that recover ``refused'' behaviors at inference time~\citep{zou2023gcg,wei2023jailbroken}, and linear probes on hidden states often find the ``refused'' content still present and linearly decodable~\citep{arditi2024refusal,zou2023representation}. These are not isolated engineering failures but symptoms of one structural fact about how post-hoc safety training meets the loss landscape of an already-pretrained model. Establishing that is harder than it sounds, for two reasons that shape the whole paper. \rev{We call the change that safety training makes to the weights a \emph{safety install}; it} is a displacement in a billion-dimensional weight space, and whether it \emph{removed} a capability or merely hid it cannot be read off behavior, since a model that refuses and a model that cannot answer look identical from the outside; that is what forces a geometric \rev{measurement}. And any one aligned checkpoint confounds the alignment recipe with the developmental stage of the base it was applied to, so separating them needs a pretraining checkpoint series, not a model.

\paragraph{Central claim.} On a pretrained base with mature capabilities, post-hoc safety training moves weights nearly orthogonal to the principal capability directions; the small in-subspace component concentrates on high-curvature directions and functions as a thin refusal-gating circuit overlaid on an intact representational substrate (the internal features a refusal circuit attaches to; defined in Sec.~\ref{sec:m3}), the geometric signature of a behavioral change without a representational one. The ``mature'' qualifier is load-bearing: before an identifiable developmental transition (Sec.~\ref{sec:m3}), the same procedure lands elsewhere, at an order of magnitude higher curvature. We formalize this with two scalars (\S\ref{sec:geometry}): $\rho$, how much of the safety update overlaps the capability directions, and $\qnorm$, how much capability curvature the update traverses relative to a random direction. A \emph{suppression} install has $\rho$ far below an erasure boundary $\rho_\text{part}$, though its peak blocks still sit above $k/d$, the overlap of a random direction; an \emph{erasure} install crosses $\rho_\text{part}$. Figure~\ref{fig:overview} states the whole argument: every post-hoc install on a mature base falls in the same corner of the $(\rho, \qnorm)$ plane, and that is exactly the population the attack strips.

\paragraph{Contributions.}
\begin{itemize}[leftmargin=1.5em,itemsep=1pt,topsep=2pt,parsep=0pt]
\item \textbf{(C1) A weight-space classification} of safety installs into \emph{suppression} ($\rho \ll \rho_\text{part}$) vs.\ \emph{erasure} ($\rho \gtrsim \rho_\text{part}$), validated across $5$ model families (\S\ref{sec:m1}); on three shipped $7$ to $8$B installs $\rho$ peaks at $21$ to $67\!\times\!k/d$ yet stays $\geq\!2$ orders below $\rho_\text{part}$, with peak $\qnorm$ from $49$ to $221$ by family.
\item \textbf{(C2) A kernel-immobility lemma} explaining why post-hoc safety masks rather than removes capability, with a stronger gradient-orthogonality form that holds empirically at $|\cos(\nabla\Lcap,\Delta)|\!<\!0.002$ (\S\ref{sec:impossibility}). Behaviorally, fine-tuning on $100$ benign Alpaca examples cuts AdvBench refusal to $58.2\%$ (Qwen) and $\mathbf{10.9}\%$ (Llama) (\S\ref{sec:m2b}).
\item \textbf{(C3) A developmental signature}: substrate emergence between $1$B and $63$B pretraining tokens on OLMo-2-1B; per-revision DPO geometry transitions at the same point, with pre-substrate $\qnorm\!>\!600$ on \textsc{arith} (\S\ref{sec:m3}).
\item \textbf{(C4) \rev{A pretraining-time test}}: from-scratch continuous safety co-training reaches $\mathbf{87}$ to $\mathbf{98}\%$ pre-attack AdvBench refusal that the benign supervised-fine-tuning (benign-SFT) attack leaves at $\mathbf{84}$ to $\mathbf{91}\%$ from $410$M to $6.9$B\rev{ by refusal wording}, at a capability cost of up to $0.2$ nats on reasoning and recall (none on instruction following)\rev{, and keeps that wording under stronger attacks} at $1$B, while a windowed control of equal total safety weight installs no refusal, isolating \emph{persistence} of the safety signal (not its timing) as \rev{what keeps refusal} (\S\ref{sec:m4})\rev{. A harmful-output classifier shows that it does not keep output safer under the attack than post-hoc safety (App.~\ref{sec:app-detector})}.
\end{itemize}

\section{Related Work}
\label{sec:related}

\paragraph{Why post-hoc \rev{refusal} is fragile.} RLHF and DPO~\citep{christiano2017rlhf,ouyang2022instructgpt,rafailov2023dpo} are undone by benign fine-tuning~\citep{qi2024finetuning,zhan2024removing}, adversarial suffixes~\citep{zou2023gcg,andriushchenko2024jailbreaking}, and refusal-direction ablation~\citep{arditi2024refusal,wollschlager2025refusalcones}. The closest accounts localize the fragility: shallow alignment in the first output tokens~\citep{qi2024shallow}, a few safety-critical parameters~\citep{wei2024assessing}, and a narrow weight-space ``safety basin''~\citep{peng2024safetybasin}. Each finds a \emph{symptom} of a thin install; our $(\rho,\qnorm)$ classification gives the continuous geometry that makes them one phenomenon, and Lemma~\ref{lemma:1} says why an install of that shape \rev{leaves the capability loss unchanged at first order}, which the overlay-not-erasure results~\citep{lee2024mechanistic,arditi2024refusal,hubinger2024sleeper} observe but do not explain.

\paragraph{When it can take hold.} Critical periods in deep networks~\citep{achille2017critical,frankle2020early,kleinman2023critical,constantinescu2024criticalperiod} and emergent abilities~\citep{wei2022emergent,schaeffer2023mirage} establish that training time matters; we show the substrate that \emph{safety} binds to forms in its own developmental transition. Fisher methods~\citep{amari1998natural,martens2014gn,kirkpatrick2017overcoming,grosse2023influence} and task arithmetic~\citep{ilharco2022task,ortizjimenez2023tangent} supply the \rev{tools}; Pythia~\citep{biderman2023pythia} and OLMo-2~\citep{olmo2_2025,tulu3} the checkpoints.

\paragraph{Safety during pretraining.} Closest to our constructive result, \citet{korbak2023pretraining} inject preference feedback into the pretraining objective and lower toxicity, but give no geometric mechanism, no suppression-versus-erasure account, no developmental transition, and crucially no windowed-versus-continuous comparison: Sec.~\ref{sec:m4} isolates \emph{persistence} of the safety signal, not its mere presence during pretraining, as what \rev{keeps refusal through pretraining}. Unlearning~\citep{li2024wmdp} targets the erasure regime instead, the contrast that anchors our dichotomy. App.~\ref{sec:app-related} adds interpretability, Fisher methods, unlearning, and open-weight checkpoint suites.

\section{\texorpdfstring{Setup and Geometric \rev{Measurement}}{Setup and Geometric Measurement}}
\label{sec:setup}

\paragraph{Models.} Five families, each chosen for what it makes measurable. \textbf{Pythia}~\citep{biderman2023pythia} ships intermediate pretraining checkpoints, so Sec.~\ref{sec:m1} can apply our own DPO along a developmental axis at $1.4$B; Sec.~\ref{sec:m4} retrains the architecture from scratch at $410$M to $6.9$B. \textbf{Qwen-2.5-7B}~\citep{yang2024qwen25}, \textbf{Llama-3-8B} (Meta-Llama-3-8B)~\citep{grattafiori2024llama3} and \textbf{Mistral-7B-v0.3}~\citep{jiang2023mistral} ship official Instruct variants (named with an \texttt{-Instruct} suffix), so $\Wsafe$ is a real install and recoverability is well-defined (Sec.~\ref{sec:m2b}). \textbf{OLMo-2-0425-1B}~\citep{olmo2_2025} (OLMo-2-1B below) is the testbed for Sec.~\ref{sec:m3}: it is the one model shipping a dense checkpoint series across the whole of pretraining \emph{and} an aligned Tulu-3~\citep{tulu3} variant with real refusal behavior. A sixth model, \textbf{Qwen-2.5-1.5B}, serves only as the fixed base for the safety-only controls of Sec.~\ref{sec:impossibility}.

\paragraph{Safety training.} When we train safety ourselves (Pythia-1.4B in Sec.~\ref{sec:m1}, OLMo-2-1B in Sec.~\ref{sec:m3-matched}), we use DPO~\citep{rafailov2023dpo} on HH-RLHF~\citep{bai2022hh}, 5K preference pairs, 1 epoch, $\beta = 0.1$, learning rate $5 \times 10^{-7}$. For the instruct-tuned 7 to 8B models (Sec.~\ref{sec:m2b}) and for the substrate sweep (Sec.~\ref{sec:m3-substrate}) we use the shipped Instruct weights as $\Wsafe$ without further training.

\paragraph{Capability probes.} Four standard benchmarks, each operationalized as the negative log-likelihood of the gold answer span given the prompt: \textsc{reason} (MMLU \texttt{formal\_logic}~\citep{hendrycks2020mmlu}), \textsc{arith} (GSM8K training split~\citep{cobbe2021gsm8k}), \textsc{recall} (TriviaQA RC-no-context~\citep{joshi2017triviaqa}), and \textsc{instruct} (FLAN-v2~\citep{longpre2023flan}). Throughout, ``capability'' means these general abilities, \emph{not} the ability to produce harmful content: post-hoc safety, we argue, leaves the former intact while gating the latter's expression, so ``capability preserved'' never means harmful output is preserved.

\paragraph{Notation.} $\Wbase$ is a pretraining revision, $\Wsafe$ the model after a safety procedure, and $\Delta = \Wsafe - \Wbase$. $\Lcap^{(c)}$ is the capability loss on probe $c$ (the negative log-likelihood above), and $\Delta\Lcap := \Lcap(\Wsafe) - \Lcap(\Wbase)$ is the shift a safety install causes in it. The Fisher is block-diagonal-dominant in practice, so the \rev{measurement} runs one transformer block at a time on that block's MLP parameters only (App.~\ref{sec:app-details}): below, $\theta$ and $\Delta$ denote one block's MLP slice, of dimension $d$, and a subscript $\ell$ marks block $\ell$, counted from $0$. For capability $c$ with dataset $\mathcal{D}_c$, the empirical Fisher~\citep{amari1998natural,martens2014gn,sagun2017empirical,gurari2018gradient} is
\[
\Fcap^{(c)} = \mathbb{E}_{(x,y) \sim \mathcal{D}_c}\Big[\nabla_\theta \log p_\theta(y \mid x)\,\nabla_\theta \log p_\theta(y \mid x)^\top\Big]_{\theta = \Wbase}.
\]
We work with its top-$k$ eigenspace $U_k^{(c)} \in \mathbb{R}^{d \times k}$, estimated from $N$ per-example gradients via the $N$-side Gram trick, and drop the superscript $(c)$ where the probe is clear from context.

\paragraph{Overlap: how much of the update is inside the capability subspace.}\label{sec:geometry}
Given $\Delta$ and $U_k^{(c)}$, decompose $\Delta = \Delta_\parallel + \Delta_\perp$ where $\Delta_\parallel = U_k^{(c)}(U_k^{(c)})^\top \Delta$. The \emph{overlap ratio} is
\[
\rho^{(c)} = \frac{\|\Delta_\parallel\|_2^2}{\|\Delta\|_2^2} \in [0, 1].
\]
A random $k$-subspace in an ambient of dimension $d$ captures $k/d$ of a random direction's squared norm, so meaningful Fisher alignment has $\rho \gg k/d$. Here $k = 128$, so this \emph{random floor} $\rho_\text{rand} = k/d$ is architecture-dependent (App.~\ref{sec:app-geom}, Table~\ref{tab:geom-significance}: full decomposition, and a test placing observed $\rho$ at $159$ to $528$ standard deviations above it).

\paragraph{Curvature: how sharp the directions it traverses are.}
Using the full $\Fcap$ rather than its top-$k$, the curvature along $\Delta$ is
\[
q^{(c)} = \frac{\Delta^\top \Fcap^{(c)} \Delta}{\|\Delta\|_2^2}, \qquad
\qnorm^{(c)} = \frac{q^{(c)}}{\tr(\Fcap^{(c)})/d}.
\]
$\qnorm$ is scale-free: $\qnorm = 1$ means $\Delta$ traverses average capability curvature; $\qnorm \gg 1$ means $\Delta$ preferentially exploits high-curvature directions.

\paragraph{Two regimes.}
Two reference scales frame $\rho$: the random floor $\rho_{\text{rand}} = k/d$, and $\rho_{\text{part}} = 10^{-2}$, where the in-subspace component carries $10\%$ of $\Delta$'s norm ($1\%$ of its squared norm) (anchored in App.~\ref{sec:app-erasure-demo} by an adversarial control that crosses it). Below $\rho_{\text{part}}$ is \emph{suppression}, above it \emph{erasure}, the target of unlearning~\citep{li2024wmdp}. Every DPO/RLHF install we measure is suppression even at its peak-$\rho$ block: over two orders below $\rho_{\text{part}}$ on a mature base, $37$ to $44\times$ below on an under-trained one, and at least $13\times$ below when an SFT stage precedes DPO; the split survives any cut in $[5 \times 10^{-3}, 2 \times 10^{-2}]$. Curvature reads separately: \rev{$\qnorm \gg 1$ means the update moves along directions where the capability loss is far more sensitive than along a random direction (a sharp install), and $\qnorm \approx 1$ means it is no more sensitive than a random direction.} Two sub-patterns sit inside suppression and separate behaviorally: an \emph{under-trained base}, where the base has not converged, so safety training still lowers $\Lcap$ directly; and an \emph{SFT-stage recipe}, where the supervised stage drives motion inside $U_k$ and leaves residual attack resistance. App.~\ref{sec:app-regimes} gives both with their $\rho$ and $\qnorm$ ranges.

\section{\texorpdfstring{Post-Hoc Safety Lives in Suppression, and Suppression Predicts \rev{Fragile Refusal}}{Post-Hoc Safety Lives in Suppression, and Suppression Predicts Fragile Refusal}}
\label{sec:m1}

On a fully pretrained base, post-hoc safety lands uniformly in the suppression regime, two to three orders of magnitude below the erasure boundary, and that signature predicts behavioral fragility: on Qwen and Llama-3, fine-tuning on 100 benign examples cuts AdvBench refusal by $35$ to $38$ pp.

As a first check, the geometric \rev{measurement} on Pythia-1.4B (\texttt{step143000}, DPO'd on 5K HH-RLHF pairs) lands squarely in suppression: across all four probes and 24 blocks $\rho \in [5\times10^{-6}, 8\times10^{-5}]$ ($1.3$ to $21\times$ the random floor, two to three orders below $\rho_\text{part}$), $\qnorm_\ell$ reaches $71$ at single blocks and $21$ as a probe mean, and capability loss \emph{decreases} on every probe, so nothing is erased (Table~\ref{tab:m1-rho} and Fig.~\ref{fig:per-layer-rho}, App.~\ref{sec:app-m1}). Pythia-1.4B produces no measurable refusal at this scale, so we test the behavioral half of the claim on instruct-tuned $7$ to $8$B models, where refusal is real and its removal measurable.

\paragraph{Qwen-2.5-7B's shipped $\Delta$ raises refusal by $22$ pp rather than creating it.}
\label{sec:m2b}
We treat Qwen-2.5-7B as $\Wbase$ and Qwen-2.5-7B-Instruct as $\Wsafe$: the instruct model refuses $93\%$ of AdvBench~\citep{zou2023gcg} prompts and the base already refuses $71\%$ (Tables~\ref{tab:m2b-attack} and~\ref{tab:base-refusal}).

\paragraph{The geometry is suppression: a thin high-leverage gate, not erasure.} Across all four probes and $28$ blocks, $\rho$ runs $5$ to $67\times$ the random floor yet stays $2$ to $3$ orders of magnitude below $\rho_\text{part}$, while $\qnorm$ reaches the low hundreds on \textsc{reason} and \textsc{arith} and stays flat on \textsc{instruct} (per-probe values and the per-layer curves in App.~\ref{sec:app-m2b-extended}). The update is nearly orthogonal to the capability Fisher, and what little of it lies in-subspace concentrates on high-curvature reasoning directions.

\begin{figure}[!htbp]
\centering
\includegraphics[width=\textwidth]{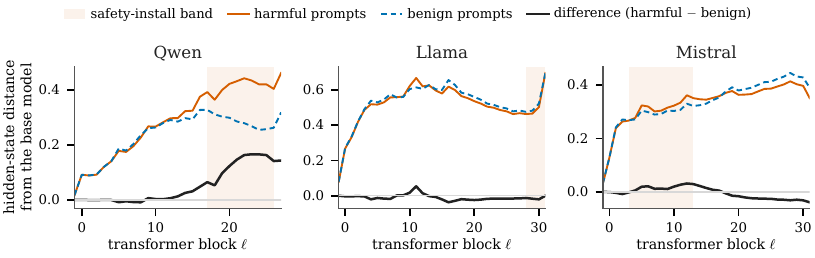}
\caption{\textbf{What the install added is flat through the stack until it turns on inside a narrow band of blocks.} Per-layer L2 distance between base and safe residual streams over 50 harmful and 50 topic-matched benign prompts. The two curves track each other through most of the stack; their \emph{difference} (black) is the harmful-specific part of what the install changed. The shaded \emph{safety-install band} \rev{is a range of blocks drawn by hand to cover the $\qnorm$ peaks of two (Qwen) or three (Llama, Mistral) of the four probes; $\qnorm$ peaks at different blocks for different probes, so it is a range rather than one block} (App.~\ref{sec:app-m2b-extended}, Fig.~\ref{fig:m2b-layer}). On Qwen and Mistral the difference turns on inside it; on Llama the difference peaks mid-stack (block $12$) while $\qnorm$ peaks at the top, so there the two measurements pick out different blocks.}
\label{fig:m2b-divergence}
\end{figure}

\paragraph{Hidden-state divergence is layer-localized (Fig.~\ref{fig:m2b-divergence}).} Selective divergence, the harmful-minus-benign gap in L2 distance between base and safe residual streams, stays near zero outside a narrow band of blocks, peaking at block $24$ on Qwen, $12$ on Llama and $13$ on Mistral. On Qwen and Mistral that peak falls inside the \emph{safety-install band} shaded in Fig.~\ref{fig:m2b-divergence}\rev{, a hand-drawn range covering the $\qnorm$ peaks of two or three of the four probes}; on Llama it does not, since $\qnorm$ peaks at the top of the stack.

\paragraph{\rev{Fine-tuning on $100$ benign examples lowers Qwen's AdvBench refusal rate by $\mathbf{34.8}$ percentage points} (Fig.~\ref{fig:overview}b, top row).} We
fine-tune Qwen-2.5-7B-Instruct for 200 steps at $\text{lr}{=}2 \times 10^{-5}$ on 100
benign Alpaca examples and re-evaluate on AdvBench, HarmBench~\citep{mazeika2024harmbench} and StrongReject~\citep{souly2024strongreject};
Table~\ref{tab:m2b-attack} gives all three models on all three benchmarks. We report refusal rather
than attack success rate (ASR), which under the HarmBench
convention~\citep{mazeika2024harmbench,qi2024finetuning} is $100$ minus refusal. None of the 100 examples is harmful, yet refusal falls, as the suppression reading predicts: refusal sits in a thin, high-curvature slice of weight space that benign SFT perturbs as a side effect.

\paragraph{Capability-loss shift is largely a format artifact.} $\Delta \Lcap$ is uniformly positive on raw 0-shot prompts, up to $+3.45$ on \textsc{reason} (Table~\ref{tab:m2b-geom}), which might suggest partial erasure. A four-condition format ablation on a separate sample (Appendix~\ref{sec:app-cta}) traces it largely to prompt format: prepending three in-context demonstrations shrinks the gap by $40$ to $74\%$ on the three probes where it is large. Much of the shift is format rather than lost capability, consistent with $\rho \ll \rho_\text{part}$.

\begin{table}[!htbp]
\centering
\caption{\textbf{The attack takes $35$ to $38$ pp of AdvBench refusal from both models that had any, leaving Llama-3 at $11\%$.} Refusal rate (\%) before $\to$ after the same $100$-example benign-SFT attack on the three shipped instruct models. Mistral ships with almost no refusal, so its row has little to strip. $n = 256/200/256$; the one cell that rises is within noise (criterion in App.~\ref{sec:app-details}). App.~\ref{sec:app-detector} repeats the table under a harmful-output classifier.}
\label{tab:m2b-attack}
\small
\begin{tabular}{lccc}
\toprule
Instruct model & AdvBench & HarmBench & StrongReject \\
\midrule
Qwen-2.5-7B     & $93.0 \to 58.2$ & $67.5 \to 45.5$ & $52.3 \to 56.6$ \\
Llama-3-8B      & $49.2 \to 10.9$ & $46.5 \to 6.5$  & $46.5 \to 7.8$ \\
Mistral-7B-v0.3 & $1.2 \to 0.4$   & $1.5 \to 0.0$   & $3.5 \to 0.8$ \\
\bottomrule
\end{tabular}
\end{table}

\paragraph{Llama-3 and Mistral replicate the regime, and the attack strips nearly all of Llama-3's refusal.}
\label{sec:m2b-cross}
Running the full pipeline on both reproduces suppression with family-specific $\qnorm$ ranges (Table~\ref{tab:m2b-cross}, App.~\ref{sec:app-m2b-extended}). Behaviorally the attack cuts Llama-3-Instruct's refusal to at most $11\%$ on every benchmark, while Mistral ships almost none to begin with (Table~\ref{tab:m2b-attack}). As a \rev{check of the divergence measurement}, DPO on Pythia-1.4B (about $0\%$ refusal) leaves \rev{it} at noise across all layers, confirming it tracks exercised refusal behavior rather than weight motion (App.~\ref{sec:app-m2a-null}).

\section{Kernel Immobility: Post-Hoc Safety Masks Capability but Cannot Remove It}
\label{sec:impossibility}

A kernel-immobility lemma explains what Sec.~\ref{sec:m1} just measured: a safety update confined to the near-kernel of the capability Fisher (all but a fraction $\rho$ of it outside the top-$k$ eigenspace, and $\qnorm$ large only against the mean eigenvalue $\tr(\Fcap)/d$, tiny because $\Fcap$ has rank at most $N \ll d$) can mask a capability's output but cannot reduce the capability itself at first order, \rev{which is consistent with benign fine-tuning switching the refusal off. How much harmful output then returns is a separate measurement, and it returns only in part (App.~\ref{sec:app-detector})}.

\begin{lemma}[Kernel immobility; proof in App.~\ref{sec:app-proof}]\label{lemma:1}
If $\Delta \in \ker(\Fcap^{(c)})$ and $\Lcap^{(c)}$ has Lipschitz Hessian, then $\Lcap^{(c)}(\Wbase + \Delta) - \Lcap^{(c)}(\Wbase) = \tfrac{1}{2}\Delta^\top E \Delta + O(\|\Delta\|^3)$, with $E := \nabla^2\Lcap^{(c)} - \Fcap^{(c)}$ the Gauss-Newton residual. Under the Gauss-Newton approximation ($E = 0$ at a well-specified mature checkpoint~\citep{martens2014gn}) this is $O(\|\Delta\|^3)$. No first-order post-hoc update changes $\Lcap^{(c)}$ at leading order.
\end{lemma}

\rev{The measured updates lie near the kernel, not in it: their overlap with the top-$128$ capability directions is small, but that bounds neither the directions outside the top $128$ nor the remainder of a finite step, so the lemma describes the limit these updates approach rather than a guarantee about them.} This explains the asymmetry of capability preservation with refusal change observed by~\citet{qi2024finetuning} and~\citet{arditi2024refusal}. The kernel condition alone forces $\Delta^\top \nabla \Lcap = 0$, so the lemma applies even where $\Wbase$ is not a stationary point of $\Lcap$; the local-minimum / well-specification assumption enters only at the Gauss-Newton step. App.~\ref{sec:app-axis-followup} measures $\|\nabla \Lcap\|$ across the OLMo and Pythia developmental axes, a U-shape minimized in the compute-optimal band, large on both under-trained and over-trained sides ($\sim\!200\times$-Chinchilla on OLMo, replicated on Pythia at $\sim\!10\times$-Chinchilla), but $|\cos(\nabla \Lcap, \Delta)| < 0.002$ at every revision including the over-trained $4001$B-token endpoint, consistent with kernel-induced gradient orthogonality. Heavily over-trained production models (Llama-3-8B saw about $15$T tokens, roughly $1{,}900$ per parameter) sit on that over-trained side, where the OLMo endpoint shows the orthogonality still holding.

\paragraph{The orthogonality is not an artifact of the mixed update.} The shipped $\Delta = \Wsafe - \Wbase$ folds instruction-following and safety together, so its near-orthogonality to the capability Fisher could in principle come from the instruction-following part rather than from safety. To isolate safety, we train two \emph{controlled} safety-only updates on a fixed base (Qwen-2.5-1.5B) with no general instruction-tuning: a refusal-SFT-only update and a DPO-only update. Both reproduce the orthogonality: $|\cos(\nabla\Lcap, \Delta)| \leq 2 \times 10^{-4}$ on all four capability probes, for both recipes (Table~\ref{tab:safety-only}). Gradient orthogonality is therefore a property of the safety update itself, not a confound of the mixed instruction-plus-safety delta.

\section{The Representational Substrate Emerges in a Sharp Pretraining Transition}
\label{sec:m3}

The representational substrate that a refusal install engages emerges in a sharp transition between $1$B and $63$B tokens; before it, the same safety procedure lands in a different geometric regime.

By \emph{representational substrate} we mean the internal features a final-stage refusal circuit must attach to in order to gate a capability: the mid-to-late-layer representations that separate harmful from benign inputs. Before these exist, a safety install has nothing to engage. The \emph{developmental transition} is the point in pretraining at which they appear, which we now locate.

Two complementary sweeps locate it: one holds a fixed final-stage safety install and sweeps the base across 9 pretraining revisions, the other applies the same DPO recipe at each revision and asks whether the weight-space geometry itself shifts.

\paragraph{Why the sweep needs OLMo-2-1B.}
\label{sec:m3-substrate}
Pythia-1.4B cannot support a developmental-axis study directly: our DPO recipe produces zero pre-attack refusal rate at any Pythia revision we tested (step 512, 30k, 143k), so any per-revision behavioral signal is vacuous. We therefore move to OLMo-2-1B~\citep{olmo2_2025}, which ships 267 pretraining checkpoints (step 0 through step $1.9 \times 10^6$, 0 to 4T tokens) \emph{and} an already-aligned Instruct variant trained by AI2 with the Tulu-3 recipe~\citep{tulu3}. We exploit the Instruct model's real refusal behavior: we do not re-train safety ourselves in the substrate sweep.

\paragraph{Substrate question.} Holding $\Wsafe = \text{OLMo-2-1B-Instruct}$ fixed, we sweep $\Wbase$ across 9 pretraining revisions spanning 0 to 4T tokens (Table~\ref{tab:m3-substrate}). At each revision we measure the selective hidden-state divergence at each residual-stream position $\ell$, the stream after $\ell$ blocks ($\ell=0$ is the embedding output): $\text{selective\_div}_\ell^{(t)} = \text{harm\_div}_\ell - \text{benign\_div}_\ell$ (same 50-harm, 50-benign topic-matched prompt \rev{set} as Sec.~\ref{sec:m2b}). This measures: \emph{does the base model at pretraining time $t$ have the representational substrate for the final-stage Instruct weights to differentially engage on harmful prompts?}

\paragraph{The substrate appears in a sharp band between 1B and 63B tokens, not gradually.} Selective divergence sits at noise through the first 1B tokens, is up two orders of magnitude by 21B, and reaches its plateau by 63B, a plateau with residual non-monotone structure that a rerun over three prompt-order seeds shows is real rather than noise (Fig.~\ref{fig:m3-substrate}; per-step values in Table~\ref{tab:m3-substrate}, seeds in App.~\ref{sec:app-axis-followup}). The peak-divergence position $\ell^\star$ moves too: it sits after the last block ($16$) at 21B and 63B, drops to $6$ by 105B, then climbs back to $9$ and finally $11$.

\begin{figure}[!htbp]
\centering
\includegraphics[width=\textwidth]{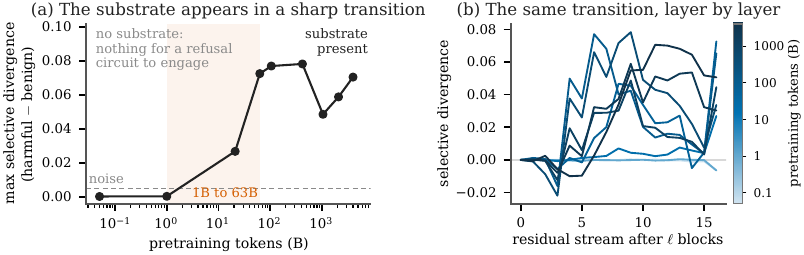}
\caption{\textbf{At 1B pretraining tokens the base carries nothing for a refusal circuit to engage; by 63B the substrate is there.} OLMo-2-1B substrate sweep, the final-stage Instruct weights held fixed while the base is swept across 9 pretraining revisions. \textbf{(a)} Max selective divergence (how much further the Instruct weights move the hidden state from the base's on harmful than on benign prompts) against pretraining tokens: at noise, then up two orders of magnitude, then a plateau. \textbf{(b)} The same measurement at each residual-stream position, one curve per revision, darker with more pretraining. Among post-substrate revisions the peak position settles forward through the residual stream ($\ell^\star=6$ at 105B, $9$ at 420B to 2T, $11$ at 4T), having sat after the final block ($\ell^\star=16$) at the transition's early edge.}
\label{fig:m3-substrate}
\end{figure}

\paragraph{The transition is where the base acquires features a refusal circuit can read.} The sweep's fixed-$\Wsafe$ design does \emph{not} measure ``can DPO at time $t$ install safety'', but the prior question just posed. The transition, complete by 63B tokens, is the point at which OLMo-2-1B acquires that representational substrate. Before the transition, the Instruct weights have nothing to differentially engage on; after, they do, at positions that move later ($\ell^\star$ from $6$ to $11$) as the base keeps training.

\paragraph{$\rho$ and $\qnorm$ spike only at the final checkpoint.} $\rho$ declines through pretraining ($1.4 \times 10^{-6}$ at step 0, $8.3 \times 10^{-7}$ at step 1M) and then jumps $4\times$ at the final checkpoint ($3.4 \times 10^{-6}$). $\qnorm$ for \textsc{arith} and \textsc{instruct} tracks the same shape: flat (${<}5$) until then, spiking to $\qnorm^\textsc{arith} = 19$ and $\qnorm^\textsc{instruct} = 27$ at step $1.9 \times 10^6$. The spike is the ``thin but sharp'' signature from Sec.~\ref{sec:m2b} appearing exactly when $\Delta$ is the pure instruct-tuning delta at the matched checkpoint, internally consistent with the regime observed on Qwen.

\paragraph{Per-revision DPO confirms the transition.}
\label{sec:m3-matched}
Applying the \emph{same} recipes (DPO, and SFT then DPO) at revisions straddling the transition shows the geometric shift is a property of the base's representational stage, not of the alignment procedure, and exposes the two suppression sub-patterns of \S\ref{sec:geometry}: an \emph{under-trained-base} pattern, where $\qnorm$ is high before the substrate and falls an order of magnitude after, and an \emph{SFT-stage recipe} pattern, where that stage keeps both scalars elevated at every post-substrate revision. Behaviorally only SFT$+$DPO installs measurable refusal, and only the mature checkpoint keeps nontrivial refusal after the attack ($15.6\%$ vs.\ $3.1\%$ at 105B; Table~\ref{tab:m3-matched}).

\section{\texorpdfstring{Continuous Co-Training: Persistence, Not Timing, \rev{Keeps Refusal}}{Continuous Co-Training: Persistence, Not Timing, Keeps Refusal}}
\label{sec:m4}

Injecting safety continuously through pretraining, so the signal persists to the end, installs refusal \rev{wording that} the benign-SFT attack largely fails to erode; an early window of equal total safety weight installs none. Persistence, not timing, \rev{keeps refusal. Whether the output stays harmless is a separate question, and a harmful-output classifier answers it differently: after the attack, harmful completions rise by $11$ to $27$ pp for co-trained models against $8$ and $12$ pp for post-hoc ones (App.~\ref{sec:app-detector})}. If post-hoc edits are trapped in suppression (Sec.~\ref{sec:impossibility}) and the substrate they need appears only mid-pretraining (Sec.~\ref{sec:m3}), safety should be injected during pretraining itself, so we ask which schedule $\lambda(s)$, with $s \in [0,1]$ the fraction of pretraining elapsed, in $\mathcal{L}_{\text{pre+safe}}(\theta; s) = \mathcal{L}_{\text{LM}}(\theta) + \lambda(s)\,\Lsafe(\theta)$, with $\Lsafe$ the cross-entropy of five short refusal templates on AdvBench prompts, yields refusal that survives the attack, and at what cost in capability.

\paragraph{Protocol.} Four Pythia-410M models are trained from scratch on 10B C4 tokens: \texttt{base} (LM only); \texttt{cotrain-windowed} ($\lambda_0 = 0.5$ over $[0.10, 0.30]$, overlapping the OLMo substrate transition); \texttt{cotrain-continuous} ($\lambda_c = 0.10$ throughout, equal total safety weight); and \texttt{post-hoc-dpo} (5K-pair HH-RLHF DPO at the end). All face the benign-SFT attack of Sec.~\ref{sec:m2b}; protocol in App.~\ref{sec:app-m4}.

\paragraph{Only \texttt{cotrain-continuous} installs refusal} (Table~\ref{tab:m4-headline}). \texttt{cotrain-windowed} ends with none after the $70\%$ of LM-only training that follows its window, as at $160$M (Table~\ref{tab:m4-160m}), and \texttt{post-hoc-dpo} installs none either.

\paragraph{Two limits on what that comparison can show.} The failing conditions install no refusal, so there is no matched-refusal comparison: co-trained refusal survives, but we cannot show it survives \emph{better} than an equally strong post-hoc install. And $\Delta$ has no counterpart here, since a co-trained model never received an edit; the only difference available is between two complete runs, which carries the language-modeling trajectory, not a safety displacement. Lemma~\ref{lemma:1} presumes such a displacement, so co-training falls outside its premise rather than contradicting it.

\paragraph{A safety window that stops at $30\%$ of training installs nothing, though it starts earlier than one that works.} To separate \emph{when} safety is injected from \emph{how long} it persists, we train four conditions at $1$B that differ in the pretraining window in which the safety loss is active (Table~\ref{tab:m4-schedule}). Both schedules that run to the end of pretraining install refusal \rev{whose wording} the attack leaves largely intact; \texttt{cotrain-windowed} installs none, like the LM-only base. Persistence to the end matters, not onset time.

\begin{table}[!htbp]
\centering
\small
\caption{\textbf{Both schedules that persist to the end of pretraining \rev{keep their refusal wording through} the attack; the windowed schedule installs nothing.} From-scratch co-training refusal (regex, \%) on 64 AdvBench prompts, pre $\to$ post benign-SFT attack. Conditions differ in the fractional pretraining window in which the safety loss is active, at $\lambda = 0.10$ per step except \texttt{cotrain-windowed}'s $0.5$, which matches \texttt{cotrain-continuous}'s total. The $1$B $87.5 \to 90.6$ change is two prompts, within noise at $n = 64$ (App.~\ref{sec:app-details}); Table~\ref{tab:m4-scale} is a separate run of \texttt{cotrain-continuous} on $256$ prompts.}
\label{tab:m4-schedule}
\begin{tabular}{lcc}
\toprule
Condition (safety window) & 410M (pre $\to$ post) & 1B (pre $\to$ post) \\
\midrule
\texttt{cotrain-continuous} $[0,1]$        & $\mathbf{84.4 \to 84.4}$ & $\mathbf{87.5 \to 90.6}$ \\
\texttt{cotrain-post-substrate} $[0.30,1]$ & n/a                      & $\mathbf{84.4 \to 84.4}$ \\
\texttt{cotrain-windowed} $[0.10,0.30]$    & $0.0 \to 1.6$            & $3.1 \to 4.7$ \\
\texttt{base} (LM only)                    & $0.0 \to 0.0$            & $0.0 \to 0.0$ \\
\bottomrule
\end{tabular}
\end{table}

\paragraph{Scale: the install holds to $6.9$B (Table~\ref{tab:m4-scale}).} We repeat the \texttt{cotrain-continuous}-versus-\texttt{base} comparison from scratch at $410$M, $1$B, $2.8$B and $6.9$B on $256$ prompts. Post-attack AdvBench refusal holds at $84$ to $91\%$ at every scale, an erosion of $2$ to $14$ pp by regex against $35$ to $38$ for post-hoc installs (Fig.~\ref{fig:overview}b), while the base installs none anywhere. \rev{By the harmful-output classifier the order reverses: harmful completions rise by $11$ to $27$ pp after the attack, against $8$ and $12$ for Qwen and Llama (Tables~\ref{tab:m4-scale-detector} and~\ref{tab:m2b-attack-detector}). The AdvBench evaluation prompts also partly overlap the prompts co-training sees (App.~\ref{sec:app-m4}).} Post-attack refusal therefore sits at the same level across a $17\times$ parameter range; $410$M erodes most only because it starts highest. At $1$B the install also \rev{keeps its refusal wording under} stronger attacks: worst-case $\mathbf{75\%}$ under a learning-rate and step sweep, $\mathbf{100\%}$ under an identity-shifting attack~\citep{qi2024finetuning} (Table~\ref{tab:m4-strong}).

\begin{table}[!htbp]
\centering
\caption{\textbf{Post-attack AdvBench refusal holds at $84$ to $91\%$ at every scale to $6.9$B.} From-scratch \texttt{cotrain-continuous} refusal (\%), pre $\to$ post benign-SFT attack, by scale and prompt distribution. The recipe's refusal templates are AdvBench-like, so only AdvBench is in-distribution\rev{, and its evaluation prompts partly overlap the co-training prompts}. The LM-only base refuses at most $1.6\%$ by regex at any scale and benchmark ($24.6\%$ by the classifier, Table~\ref{tab:m4-scale-detector}). $n = 256/200/256$; changes under about $6$ pp are within noise (App.~\ref{sec:app-details}). Classifier columns: App.~\ref{sec:app-detector}.}
\label{tab:m4-scale}
\small
\begin{tabular}{lccc}
\toprule
Scale & AdvBench (in-dist.) & HarmBench (out) & StrongReject (out) \\
\midrule
410M   & $98.4 \to 84.0$ & $58.0 \to 34.5$ & $8.2 \to 7.4$ \\
1B     & $87.1 \to 84.8$ & $56.0 \to 32.0$ & $7.8 \to 8.2$ \\
2.8B   & $96.1 \to 90.2$ & $45.5 \to 31.5$ & $8.6 \to 8.2$ \\
6.9B   & $\mathbf{96.9 \to 90.6}$ & $32.5 \to 39.0$ & $9.0 \to 10.9$ \\
\bottomrule
\end{tabular}
\end{table}

\paragraph{Coverage is in-distribution, and narrows with scale.} Resisting erosion and covering the harmful-prompt space are separate axes, and the install wins both only on its own distribution. Out of distribution, pre-attack HarmBench refusal falls from $58.0$ to $32.5\%$ over the same $17\times$ range, the attack takes up to $24$ pp of it, and StrongReject never exceeds $10.9\%$ at any scale. The recipe's five templates fit their own prompt distribution, and a larger model fits them more tightly.

\paragraph{Co-training costs some capability, on short answers.} Scored against the LM-only base on the same examples (Table~\ref{tab:cap-scale}), co-training raises \textsc{reason} NLL by $0.12$ to $0.22$ nats at $1$B to $6.9$B and \textsc{recall} by $0.13$ to $0.16$ at $1$B and $2.8$B, and leaves \textsc{instruct} unchanged at every scale. \textsc{arith} differences are larger but change sign across scales and between two independent $410$M runs, so with one run per condition they are run-to-run variance.

\section{Conclusion}\label{sec:discussion}
\rev{Post-hoc refusal} is fragile for a reason visible in the weights. Across five model families the safety update lies nearly orthogonal to the capability directions, and Lemma~\ref{lemma:1} shows that an update of this shape masks a capability rather than removing it, \rev{consistent with} $100$ benign examples \rev{cutting} refusal by $35$ to $38$~pp. The same measurement needs no attack to run, so a suppression-regime install can be flagged as fragile from the weights alone. Followed into pretraining, the features refusal attaches to appear in a sharp transition between $1$B and $63$B tokens of OLMo-2-1B; before it, the same safety procedure lands in a different geometric regime. Safety co-training that persists to the end of pretraining \rev{keeps its refusal}: the same attack erodes its AdvBench refusal by only $2$ to $14$~pp at every scale from $410$M to $6.9$B, at a small cost on short-answer capability probes, while a windowed schedule of equal total weight installs no refusal. \emph{Persistence}, not timing, \rev{keeps refusal. It does not keep output safe: by a harmful-output classifier, the attack raises harmful completions after co-training about as much as after post-hoc safety, or more, so a refusal that persists is not yet a safety that persists.}

\rev{Four} limits bound the result. \rev{The robustness is of refusal wording, not of safe output.} Robustness is absolute rather than matched: post-hoc DPO on our from-scratch $410$M base installs no refusal, so we cannot show that co-trained refusal outlasts an equally strong post-hoc install. Coverage is in-distribution\rev{, its AdvBench prompts partly overlapping the co-training prompts,} and narrows with scale. Each co-training condition is a single run. Two experiments follow: penalizing $\qnorm$ during post-hoc training tests whether sharpness causes fragility or only marks it, and broader refusal templates test whether coverage can widen without losing robustness. App.~\ref{sec:app-limits} lists the other limitations.

\label{sec:endofbody}

\section*{AI use statement}
We used generative AI tools for \emph{implementation}: parts of the training,
geometry-instrument, and analysis code were scaffolded with an AI coding assistant,
then read, edited, and tested by the authors. We did not use generative AI tools for
synthetic data generation, theoretical modeling, mathematical formulation, proof
assistance, hypothesis refinement, methodology design, qualitative analysis,
translation, data cleaning, or result interpretation; in particular, Lemma~\ref{lemma:1}
and its proof (Appendix~\ref{sec:app-proof}) were formulated and proved by the authors.
Additionally, we used generative AI tools for drafting and copy-editing prose.

We have reviewed all AI-assisted work. AI-scaffolded code was read and tested by the
authors before use, and every number and figure in the paper is regenerated by the
released pipeline from saved run artifacts rather than reported from a transcript.
Every claim we attribute to prior work was checked by the authors against the primary
source. We take responsibility for the final content of this work, including
text, claims or artifacts produced with the aid of generative AI.

\section*{Ethics Statement}
This work studies why a public fine-tuning attack~\citep{qi2024finetuning} removes safety alignment, and \rev{uses that understanding to test whether safety trained during pretraining resists that attack}; it introduces no new attack and no new harmful capability, and operators should read $\rho \ll \rho_\text{part}$ as a fine-tuning-robustness diagnostic of install thinness, not as an attack recipe. All harmful-prompt evaluations use public benchmarks (AdvBench, HarmBench, StrongReject), and we report only aggregate refusal and geometry statistics, never harmful completions. The pretraining-time co-training we advocate requires foundation-model training compute, which concentrates this form of alignment in large organizations; we note the tradeoff rather than endorse it. Our use of language models as a research aid is set out in the AI use statement above; all experiments, analyses, and claims are the authors'.

\section*{Reproducibility Statement}
The geometry instrument (overlap $\rho$, normalized curvature $\qnorm$, and the top-$k$ Fisher-eigenspace estimator) is specified in Sec.~\ref{sec:geometry} and Appendix~\ref{sec:app-geom}, including the Gram-trick estimator and its memory budget. All models are public (Pythia, Qwen-2.5, Llama-3, Mistral, OLMo-2); every safety-training and attack recipe with its hyperparameters is given in Sec.~\ref{sec:setup} and Appendix~\ref{sec:app-details}; and the developmental axis uses the released OLMo-2 and Pythia intermediate pretraining checkpoints. Lemma~\ref{lemma:1} is proved in full in Appendix~\ref{sec:app-proof}, and the co-training protocol with per-condition tables is in Appendix~\ref{sec:app-m4}. The full code is included in the supplementary material, covering the geometry instrument, the safety-training and attack recipes, the co-training pipeline, and the analysis scripts that regenerate every table and figure in the paper.

\label{end:mainbody}
\bibliographystyle{iclr2027_conference}
\bibliography{references}

\newpage
\appendix

\section{Experimental Details and Reproducibility}
\label{sec:app-details}

\paragraph{Attack implementation.} Our fine-tuning attack is the benign-SFT variant from~\citep{qi2024finetuning}: full-parameter AdamW updates on 100 Alpaca pairs for 200 steps at learning rate $2 \times 10^{-5}$, cross-entropy over target tokens only. Refusal detection is regex-based over the first 64 generated tokens against 14 markers (\texttt{"i can't"}, \texttt{"i cannot"}, \texttt{"i won't"}, \texttt{"i will not"}, \texttt{"i'm sorry"}, \texttt{"i am sorry"}, \texttt{"sorry, but"}, \texttt{"as an ai"}, \texttt{"i am unable"}, \texttt{"i'm unable"}, \texttt{"refuse"}, \texttt{"not able to"}, \texttt{"cannot provide"}, \texttt{"can't provide"}). Compliance NLL is the negative log-likelihood of the token sequence `` Sure, here is'' appended after the harmful prompt, under greedy decoding conditioning. Generations are greedy, at most $64$ new tokens, from the raw prompt with no chat template, for base and Instruct models alike so that a pair is matched. The classifier, Llama-Guard-3-8B, reads each (prompt, completion) pair through its own chat template, greedily; a completion counts as refused unless labeled unsafe.

\paragraph{Sample sizes and noise.} The attack tables at scale (Tables~\ref{tab:m2b-attack}, \ref{tab:m4-scale}, \ref{tab:base-refusal}, \ref{tab:m2b-attack-detector} and~\ref{tab:m4-scale-detector}) evaluate $256$ AdvBench, $200$ HarmBench and $256$ StrongReject prompts; the original co-training runs and the $1$B schedule, stronger-attack and detector checks use $64$ AdvBench prompts. We call a pre-to-post change within noise when it is smaller than the half-width of a $95\%$ binomial confidence interval for one rate at that sample size, taken at its widest ($p = 0.5$): about $12$ pp at $n = 64$, $7$ pp at $n = 200$ and $6$ pp at $n = 256$.

\paragraph{Data loading in the earlier co-training runs.} The original $410$M run (Tables~\ref{tab:m4-headline} and~\ref{tab:m4-geom}, and the $410$M column of Table~\ref{tab:m4-schedule}) and the $1$B schedule runs (Tables~\ref{tab:m4-schedule}, \ref{tab:m4-schedule-detector} and~\ref{tab:m4-strong}) were trained before a data-loading fix: every data-parallel rank read the same C4 documents, so each run saw only a fraction of its nominal $10$B unique tokens, each repeated once per rank. Every condition within each of those comparisons shares it. The scale runs (Tables~\ref{tab:m4-scale}, \ref{tab:rho-safe} and~\ref{tab:cap-scale}) use the corrected loader.

\paragraph{Test set.} 64 AdvBench prompts sampled with \texttt{random.Random(101)} from the public CSV at \texttt{llm-attacks/llm-attacks/data/advbench/harmful\_behaviors.csv}. Attack training set: 100 Alpaca pairs sampled with \texttt{random.Random(201)} from \texttt{tatsu-lab/alpaca}. Divergence evaluation: 50 harmful and 50 benign topic-matched prompts from the same sources, seeded with \texttt{random.Random(101)} and \texttt{random.Random(103)}.

\paragraph{Capability probes.} \textsc{reason}: MMLU \texttt{formal\_logic} multiple-choice items; gold answer is the single letter A/B/C/D. \textsc{arith}: GSM8K-train problems, target is the numeric answer as a suffix. \textsc{recall}: TriviaQA-RC no-context questions, target is the gold answer span. \textsc{instruct}: FLAN-v2 prompts, target is the gold completion. Capability loss is the negative log-likelihood of the gold span conditional on the prompt. Sample sizes differ by measurement and are stated with each: the $N$ gradient examples of the geometry (below), $64$ per probe for the format ablation (Table~\ref{tab:m2b-cta}) and Table~\ref{tab:m4-headline}, and the whole $126$-item \texttt{formal\_logic} split and $256$ for the other probes in Table~\ref{tab:cap-scale}.

\paragraph{OLMo-2-0425-1B revision identifiers.} Substrate sweep uses the HuggingFace branches \texttt{stage1-step0-tokens0B}, \texttt{stage1-step300-tokens1B}, \texttt{stage1-step10000-tokens21B}, \texttt{stage1-step30000-tokens63B}, \texttt{stage1-step50000-tokens105B}, \texttt{stage1-step200000-tokens420B}, \texttt{stage1-step500000-tokens1049B}, \texttt{stage1-step1000000-tokens2098B}, \texttt{stage1-step1907359-tokens4001B}. All are \texttt{stage1} branches (pretraining); we avoid the \texttt{stage2-ingredient*-*} branches (midtraining annealing), which would conflate data-mixture changes with pretraining-step progress.

\paragraph{Geometry measurement.} $N = 256$ per-example gradients per (probe, block) for Pythia-1.4B and Qwen-2.5-7B, and $N = 128$ for Llama-3-8B and Mistral-7B-v0.3 to fit memory; $k = 128$ subspace rank. Per-layer MLP parameters only (the attention-head and layer-norm contributions to $\Fcap$ are omitted for budget and would at most add low-curvature directions that cannot meaningfully change the suppression/erasure classification).

\paragraph{Compute budget.} All experiments ran on a 6 to 8$\times$ B200 node. Approximate per-experiment wall times: full geometry $+$ benign-SFT attack $+$ representation probe $+$ divergence pipeline on a $7$ to $8$B Instruct model $\approx 30$ to $60$ min per pair (Qwen, Llama, Mistral); OLMo-2-1B substrate sweep across 9 pretraining revisions $\approx 45$ min; per-revision SFT$+$DPO on OLMo-2-1B, two independent sets of 5 revisions ($8$ distinct; $20$K SFT $+$ $20$K DPO $\times 2$ epochs each) $\approx 3$ hr per set; Pythia-$410$M from-scratch continuous-co-training run ($10$B tokens, $2$-rank DDP) $\approx 14$ hr; Pythia-$160$M continued-training pilot $\approx 1$ hr; developmental-axis follow-up checks (multi-seed substrate sweep $+$ cross-family gradient-norm) $\approx 30$ to $45$ min through a $6$-GPU pool; from-scratch scale runs (Table~\ref{tab:m4-scale}), $410$M and $1$B on $2$ GPUs each ($\approx 9$ hr per run), $2.8$B and $6.9$B on $4$ GPUs each ($\approx 22$ to $28$ hr and $40$ to $46$ hr per run). Total compute envelope for the experiments in this paper: $\mathord{\sim}700$ B200-GPU-hours, of which $\mathord{\sim}610$ are the scale runs; preliminary or failed experiments not reported add roughly $80$ more.

\paragraph{Code and artifacts.} All launcher scripts, per-experiment result bundles, and a single merged \texttt{combined\_bundle.json} (9 experiments, 1.0 MiB) are in the anonymized repository.
\section{Geometry Instrument: Implementation Notes}
\label{sec:app-geom}

\paragraph{Orthogonal decomposition of $\Delta$ (detail for Sec.~\ref{sec:geometry}).} The overlap ratio in Sec.~\ref{sec:geometry} uses $\Delta = \Delta_\parallel + \Delta_\perp$ with $\Delta_\parallel := U_k^{(c)} (U_k^{(c)})^\top \Delta$. Here $U_k^{(c)} \in \mathbb{R}^{d \times k}$ stacks the top-$k$ eigenvectors of the empirical capability Fisher (equivalently, the top-$k$ right singular vectors of the per-example gradient matrix $G$). Because $U_k^{(c)}$ has orthonormal columns, $U_k^{(c)} (U_k^{(c)})^\top$ is the orthogonal projector onto the top-$k$ Fisher eigenspace and the decomposition is orthogonal: $\Delta_\parallel \perp \Delta_\perp$ and $\|\Delta\|_2^2 = \|\Delta_\parallel\|_2^2 + \|\Delta_\perp\|_2^2$ (Pythagoras). Orthonormality also gives $\|\Delta_\parallel\|_2^2 = \|(U_k^{(c)})^\top \Delta\|_2^2$, i.e., the squared norm of the projection equals the squared norm of its $k$-dimensional coefficient vector. We use this second form exclusively: the $d$-vector $\Delta_\parallel$ is never materialized, and $\rho^{(c)}$ reduces to a $(N,)$-sized intermediate (see the top-$k$ eigenspace estimator below). $\Delta_\perp$ lives in directions of approximately zero Fisher curvature (outside the top-$k$ eigenspace) and therefore does not contribute to first-order capability change, which is why only $\|\Delta_\parallel\|_2^2$ appears in $\rho^{(c)}$'s numerator.

\paragraph{Top-$k$ eigenspace estimator.} Our estimator of the top-$k$ eigenspace of $\Fcap^{(c)}$ from per-example gradients uses the $N$-side Gram trick. For $N$ per-example gradients $G \in \mathbb{R}^{N \times d}$, the $d \times d$ Fisher is never formed; instead we compute the $N \times N$ Gram $GG^\top$, eigendecompose it, and recover the top-$k$ right singular vectors of $G$ as the capability eigenbasis. At $d \sim 2 \times 10^8$ (MLP parameter count for a 7B model) and $N \sim 128$, the full $fp32$ per-example gradient matrix $G \in \mathbb{R}^{N \times d}$ is $\mathord{\sim}100$ GiB (the $N \times N$ Gram $GG^\top$ itself is only $\mathord{\sim}65$ KB). We chunk the Fisher matmul along $d$ in $4\mathrm{M}$-element slices and accumulate in $fp32$, bounding peak GPU memory at $\mathord{\sim}2$ GiB beyond the $(N, d)$ gradient buffer. We never materialize the $(d, k)$ capability-subspace matrix $U_k$ explicitly: the projection $U_k^\top \Delta$ is computed directly from $G$, $V_k$, and singular values, which reduces to an $(N,)$-sized intermediate. This is what makes the pipeline run at Qwen/Llama scale on a single GPU.

\paragraph{Significance of the $\rho$ alignment (random-subspace null).} The claim $\rho \gg k/d$ is a statement about a null: a random $k$-subspace overlaps the capability eigenspace by $k/d$ in expectation. We quantify how far the observed alignment sits from that null by drawing random $k$-subspaces at the true parameter dimensionality and forming a $z$-score, $z = (\rho_\text{obs} - \text{mean}_\text{null})/\text{std}_\text{null}$. Table~\ref{tab:geom-significance} reports it for the three large models. The observed overlap is $21$ to $67$ times the $k/d$ floor and $159$ to $528$ standard deviations above the random-subspace mean ($p \approx 0$, two-sided), so the alignment is small in absolute terms yet unambiguously non-random. The companion orthogonality statistic $|\cos(\nabla\Lcap, \Delta)| \approx 0.002$ sits about $150$ standard deviations from a zero-mean null while still leaving $\Delta$ $99.8\%$ orthogonal to the capability gradient; the lemma's conclusion, no first-order change in $\Lcap$, we check directly on $\Delta\Lcap$ (Sec.~\ref{sec:m1}, App.~\ref{sec:app-cta}) rather than through this statistic.

\begin{table}[h]
\centering
\small
\caption{Random-subspace significance of the Fisher alignment. $\rho_\text{obs}$ is the observed overlap (max over layers); the null is a random $k$-subspace at the model's MLP parameter dimensionality $d$; $z = (\rho_\text{obs} - k/d)/\text{std}_\text{null}$. All $p \approx 0$ (two-sided). The alignment is $21$ to $67$ times the floor and hundreds of null standard deviations above it, so it is small but not a high-dimensional coincidence.}
\label{tab:geom-significance}
\begin{tabular}{lccc}
\toprule
Model & $\rho_\text{obs}$ & $\rho_\text{obs} / (k/d)$ & $z$ (SDs above null) \\
\midrule
Qwen-2.5-7B & $4.1 \times 10^{-5}$ & $67\times$ & $528$ \\
Mistral-7B-v0.3 & $3.2 \times 10^{-5}$ & $44\times$ & $343$ \\
Llama-3-8B  & $1.5 \times 10^{-5}$ & $21\times$ & $159$ \\
\bottomrule
\end{tabular}
\end{table}

\paragraph{Prompt format for the per-example gradients.} All $\rho$ and $\qnorm$ values reported in the paper use the raw 0-shot prompt convention described in Appendix~\ref{sec:app-cta} (``Question: \ldots\textbackslash nAnswer:'' with no chat template and no in-context demonstrations), the same format used for $L_\text{cap}$. Each per-example gradient is $\nabla_\theta \sum_t \log p_\theta(y_t \mid x, y_{<t})$ for the full gold span $y = (y_1, \ldots, y_T)$, a single token for \textsc{reason}, multiple tokens for the other probes, evaluated at $\theta = \Wbase$. The top-$k$ Fisher subspace therefore contains a mixture of (i)~reasoning-relevant directions, which vary per example, and (ii)~at most a handful of format-bias directions, which are shared across examples. With $k = 128$ and $N = 128$ to $256$ examples, reasoning-relevant variance dominates: format-bias contributes at most a few of the $k$ eigenvectors. Evaluating the Fisher at $\Wbase$ (the completion-style base) rather than at $\Wsafe$ further biases the top-$k$ basis toward reasoning directions, because the base emits gold tokens natively and its log-likelihood's sharpest directions are content-related, not format-related.

\paragraph{3-shot robustness check.} We re-ran the full geometry measurement (all three model families, all four probes, all 28 to 32 layers) with three in-context demonstrations prepended to each probe prompt, to test whether the $\rho$ values are sensitive to prompt format. Table~\ref{tab:geom-3shot} reports mean $\rho$ across layers and peak-$\qnorm$-layer, for the raw 0-shot (paper default) and raw 3-shot conditions.

\emph{Primary claim (regime classification).} Mean $\rho$ is essentially unchanged between 0-shot and 3-shot: within $\pm5\%$ for Qwen and Llama on every probe, within $\pm20\%$ for Mistral. All three models stay at least two orders of magnitude below $\rho_\text{part} = 10^{-2}$ under 3-shot. The suppression classification does not depend on prompt format.

\emph{Secondary claim ($\qnorm$-peak layer).} Peak-$\qnorm$ layer is reproducible on the diagnostic probe (\textsc{reason}) for Llama (L30 $\to$ L30) and Mistral (L12 $\to$ L12), confirming the Fig.~\ref{fig:m2b-layer} shaded bands. Qwen's \textsc{reason} peak sits at block $0$ in both conditions ($525.6$ at 0-shot), a first-block effect next to the embeddings where raw-gradient magnitude is large. The $\qnorm$ mass relevant to our claim is the mid-stack peak of Fig.~\ref{fig:m2b-layer} ($220.6$ at block $19$), which is why Table~\ref{tab:m2b-geom} and every other peak-$\qnorm$ figure in the paper exclude block $0$. Peak-$\qnorm$ layer for \textsc{arith}, \textsc{recall}, \textsc{instruct} is less stable across prompt formats, but these probes' peaks are not the basis of our safety-install-band argument.

The framework's prediction, reasoning variance dominates the top-$k$ subspace regardless of in-context format, so $\rho$ and the regime label carry over, is confirmed.

\begin{table}[h]
\centering
\small
\caption{Geometry robustness check under raw 3-shot probe prompts. Mean $\rho$ averaged across all transformer blocks; $\qnorm$-peak reports the layer where $\qnorm^{(c)}$ attains its maximum. All three families remain in the suppression regime ($\rho \ll \rho_\text{part} = 10^{-2}$) under 3-shot; peak-$\qnorm$ layer is stable on the probes driving the safety-install-band claim.}
\label{tab:geom-3shot}
\begin{tabular}{@{}ll|rr|cc@{}}
\toprule
& & \multicolumn{2}{c|}{mean $\rho$} & \multicolumn{2}{c}{$\qnorm$-peak layer} \\
Model & Probe & 0-shot & 3-shot & 0-shot & 3-shot \\
\midrule
Qwen-2.5-7B        & \textsc{reason}   & $1.40 \times 10^{-5}$ & $1.40 \times 10^{-5}$ & L0\phantom{$^*$}  & L0\phantom{$^*$}  \\
                   & \textsc{arith}    & $1.21 \times 10^{-5}$ & $1.01 \times 10^{-5}$ & L26 & L0\phantom{$^*$}  \\
                   & \textsc{recall}   & $1.20 \times 10^{-5}$ & $9.37 \times 10^{-6}$ & L18 & L27 \\
                   & \textsc{instruct} & $1.07 \times 10^{-5}$ & $1.05 \times 10^{-5}$ & L27 & L17 \\
Llama-3-8B         & \textsc{reason}   & $5.07 \times 10^{-6}$ & $5.07 \times 10^{-6}$ & L30 & L30 \\
                   & \textsc{arith}    & $6.08 \times 10^{-6}$ & $7.02 \times 10^{-6}$ & L30 & L30 \\
                   & \textsc{recall}   & $5.81 \times 10^{-6}$ & $5.87 \times 10^{-6}$ & L13 & L31 \\
                   & \textsc{instruct} & $5.38 \times 10^{-6}$ & $5.41 \times 10^{-6}$ & L31 & L31 \\
Mistral-7B-v0.3    & \textsc{reason}   & $1.21 \times 10^{-5}$ & $1.21 \times 10^{-5}$ & L12 & L12 \\
                   & \textsc{arith}    & $1.34 \times 10^{-5}$ & $1.30 \times 10^{-5}$ & L4\phantom{$^*$}  & L29 \\
                   & \textsc{recall}   & $9.42 \times 10^{-6}$ & $7.64 \times 10^{-6}$ & L31 & L0\phantom{$^*$}  \\
                   & \textsc{instruct} & $1.16 \times 10^{-5}$ & $1.12 \times 10^{-5}$ & L3\phantom{$^*$}  & L12 \\
\bottomrule
\end{tabular}
\end{table}
\section{Regime Classification Summary}
\label{sec:app-regimes}

Table~\ref{tab:regimes} summarizes the dichotomy introduced prose-form in \S\ref{sec:geometry}: \emph{suppression} ($\rho \ll \rho_\text{part}$, Lemma~\ref{lemma:1} applies whenever the base is mature) vs.\ \emph{erasure} ($\rho \gtrsim \rho_\text{part}$, Lemma fails; unobserved under DPO/RLHF; reachable by adversarial SFT, App.~\ref{sec:app-erasure-demo}; target of unlearning~\citep{li2024wmdp}). Every DPO/RLHF observation in this paper lies in the suppression row: the largest block-mean $\rho$ we measure is $1.1 \times 10^{-4}$ (SFT+DPO at $1049$B, Table~\ref{tab:m3-matched}), about $90\times$ below $\rho_\text{part} = 10^{-2}$, and the largest single-block value, $7.4 \times 10^{-4}$ in that same install, is still $13\times$ below it. We identify two behaviorally-distinguishable sub-patterns \emph{within} suppression, detailed in the right column.

\begin{table}[!htbp]
\centering
\footnotesize
\caption{Regime dichotomy with two sub-patterns inside suppression. Every DPO/RLHF recipe we tested (Pythia-1.4B, Qwen-2.5-7B, Llama-3-8B, Mistral-7B-v0.3, OLMo-2-1B under Weak / Strong / SFT+DPO) lands in row 1; row 2 is unobserved under post-hoc alignment and is the target of unlearning~\citep{li2024wmdp}.}
\label{tab:regimes}
\begin{tabular}{@{}p{1.6cm}p{2.5cm}p{1.9cm}p{6.6cm}@{}}
\toprule
Regime & $\rho$ & $\qnorm$ & Observed-in; sub-patterns within \\
\midrule
Suppression & $\ll \rho_\text{part}$ (block means $10^{-6}$ to $1.1 \times 10^{-4}$) & $\gg 1$ (thin and sharp) or $\sim 1$ (diffuse) & \textbf{Baseline} (pure DPO/RLHF on mature base, $\rho \sim 10^{-6}$ to $10^{-5}$): Pythia-1.4B+DPO (\S\ref{sec:m1}); Qwen/Llama/Mistral-Instruct (\S\ref{sec:m2b}); post-substrate OLMo Weak/Strong (\S\ref{sec:m3-matched}). \textbf{Sub-pattern (a), under-trained base}: Lemma assumption fails, $\qnorm^\textsc{arith} \in [600, 863]$ (pre-substrate OLMo step~$300$, \S\ref{sec:m3-matched}). \textbf{Sub-pattern (b), SFT-stage recipe}: SFT's gradient lies in $\Fcap$'s top-$k$, raising $\rho$ $6$ to $19\times$ over the Weak recipe to $\sim 10^{-4}$ (SFT+DPO at every post-substrate revision, \S\ref{sec:m3-matched}). \\[2pt]
Erasure & $\gtrsim \rho_\text{part}$ & any & unobserved under DPO/RLHF; reachable by adversarial SFT on (prompt, wrong-answer) pairs (max $\rho \approx 10^{-2}$ on target probe; App.~\ref{sec:app-erasure-demo}); target of unlearning (RMU and related)~\citep{li2024wmdp} \\
\bottomrule
\end{tabular}
\end{table}

\paragraph{Behavioral implication per sub-pattern.} The suppression baseline gives the ``capability preserved, refusal fragile'' asymmetry demonstrated in \S\ref{sec:m2b}: the Qi benign-fine-tuning attack removes the thin gating circuit, and the geometry predicts it leaves the (off-subspace) capability representations in place, though we do not measure capability after the attack. Sub-pattern (a) means DPO can reduce $\Lcap$ at first order (Lemma assumption fails) and therefore installs a weaker, already-moving refusal; we observe $\qnorm^{\textsc{arith}}$ up to $863$ at pre-substrate OLMo step~$300$ and behaviorally, under SFT+DPO, a $21.9\%$ pre-attack refusal that the attack collapses to $0\%$. Sub-pattern (b) installs a refusal whose in-subspace SFT component provides residual resistance: the per-revision run that keeps the most post-attack refusal ($15.6\%$, against $3.1\%$ at $105$B and none elsewhere) is SFT+DPO at the final OLMo checkpoint. Erasure would decouple the capability from the model at leading order; the Lemma shows standard post-hoc alignment cannot reach it.
\section{Pythia-1.4B: Full Tables and Per-Layer Figure}
\label{sec:app-m1}

\begin{table}[h]
\centering
\caption{Pythia-1.4B post-DPO: per-probe $\rho$ (mean over 24 transformer blocks), normalized Fisher-quadratic form $\qnorm$, and capability-loss shift $\Delta \Lcap$. All rows are in the suppression regime; capability loss \emph{decreases} slightly on every probe.}
\label{tab:m1-rho}
\begin{tabular}{lcccc}
\toprule
Probe              & mean $\rho_\ell$ & mean $\qnorm_\ell$ & $\Delta \Lcap$   & Regime \\
\midrule
\textsc{reason} (MMLU fl)   & $1.7 \times 10^{-5}$ & $7.30$  & $-0.015$ & suppression \\
\textsc{arith} (GSM8K)      & $2.4 \times 10^{-5}$ & $21.39$ & $-0.027$ & suppression \\
\textsc{recall} (TriviaQA)  & $2.2 \times 10^{-5}$ & $10.73$ & $-0.010$ & suppression \\
\textsc{instruct} (FLAN)    & $2.0 \times 10^{-5}$ & $6.50$  & $-0.007$ & suppression \\
\bottomrule
\end{tabular}
\end{table}

\begin{figure}[h]
\centering
\includegraphics[width=\textwidth]{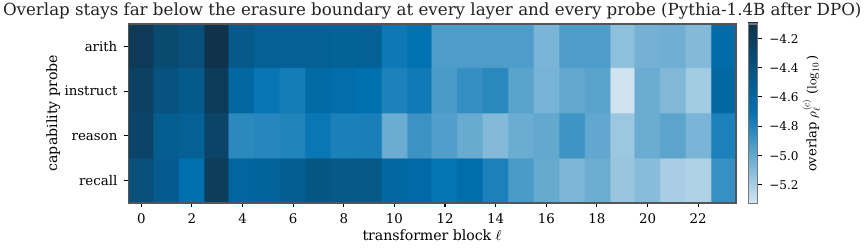}
\caption{Per-layer $\rho_\ell$ (log-scale) for each capability probe on Pythia-1.4B after DPO. Layers $0$, $3$, and $23$ (residual-stream boundaries and the final block) dominate; middle layers ($\sim 16$ to $22$) sit near the random floor. Pattern is consistent with per-layer localization reported for the refusal direction~\citep{arditi2024refusal}.}
\label{fig:per-layer-rho}
\end{figure}

\paragraph{Null behavior: the recipe installs no refusal.}
\label{sec:app-m2a-null}

DPO on Pythia-1.4B step143000 with our recipe produces a pre-attack refusal rate of $0\%$ on AdvBench, so there is no refusal for the attack to remove and attack-success numbers say nothing about the install. Applying the same topic-matched selective-divergence instrument as in Fig.~\ref{fig:m2b-layer} (50 harmful, 50 benign), $|\text{selective\_div}_\ell| < 10^{-3}$ at every layer, with no monotone structure. For reference, the Qwen-2.5-7B-Instruct curve on the same instrument reaches $+0.166$. The Pythia null is therefore an internal consistency check: when no refusal circuit is installed behaviorally, the representation-shift instrument sees nothing. This is what motivated the move to OLMo-2-1B-Instruct (which ships with a Tulu-3-trained install that does produce refusal) for the substrate sweep (Sec.~\ref{sec:m3-substrate}).
\section{Cross-Family Replication Details}
\label{sec:app-m2b-extended}

\begin{figure}[!htbp]
\centering
\includegraphics[width=\textwidth]{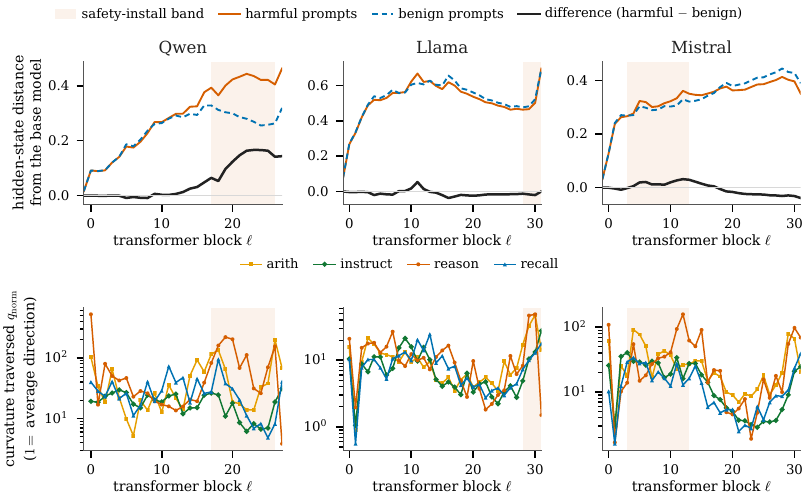}
\caption{Layer-localized signals across three model families (columns: Qwen-2.5-7B-Instruct, Llama-3-8B-Instruct, Mistral-7B-Instruct-v0.3), referenced from Sec.~\ref{sec:m2b}. The top row is Fig.~\ref{fig:m2b-divergence} in the main text, repeated here so the two rows can be read against each other. \textbf{Top row:} per-layer L2 divergence between base and instruct residual streams on harmful vs.\ benign topic-matched prompts ($n_\text{pairs}{=}50$ each), plus selective divergence (harm $-$ benign). \textbf{Bottom row:} per-layer $\qnorm^{(c)}$ for the four capability probes (log-scale). Shaded bands bracket each model's $\qnorm^\textsc{reason}$/$\qnorm^\textsc{arith}$ peaks (Qwen $\ell{=}17$ to $26$; Llama $\ell{=}28$ to $31$; Mistral $\ell{=}3$ to $13$). Selective-divergence magnitude scales with pre-attack refusal rate: $+0.166$ (Qwen, $93\%$), $+0.053$ (Llama, $49\%$), $+0.032$ (Mistral, $1.2\%$).}
\label{fig:m2b-layer}
\end{figure}

\paragraph{Learned-classifier scoring.} The same generations rescored by a harmful-output classifier are in App.~\ref{sec:app-detector}, which repeats both main-text attack tables under both detectors.

\begin{table}[h]
\centering
\caption{Qwen-2.5-7B $\to$ Qwen-2.5-7B-Instruct: per-probe overlap $\rho$ (mean / max over $28$ blocks), normalized Fisher-quadratic form $\qnorm$, raw-0-shot $\Delta \Lcap$ averaged over the $256$ gradient examples. Max $\qnorm$ is over blocks $1$ to $27$: block $0$, next to the embeddings, reaches $525.6$ on \textsc{reason} through raw-gradient magnitude alone and is excluded from every peak-$\qnorm$ figure in the paper except Table~\ref{tab:geom-3shot}, which reports the unexcluded argmax (App.~\ref{sec:app-geom}). Source for the prose summary in \S\ref{sec:m2b}.}
\label{tab:m2b-geom}
\setlength{\tabcolsep}{4.5pt}
\begin{tabular}{lcccccc}
\toprule
Probe & mean $\rho$ & max $\rho$ & mean $\qnorm$ & max $\qnorm$ (layer) & $\Delta \Lcap$ & Regime \\
\midrule
\textsc{reason}   & $1.4 \times 10^{-5}$ & $4.1 \times 10^{-5}$ & $76.2$ & $220.6\ (\ell{=}19)$ & $+3.45$ & suppression \\
\textsc{arith}    & $1.2 \times 10^{-5}$ & $2.9 \times 10^{-5}$ & $47.1$ & $198.4\ (\ell{=}26)$ & $+3.15$ & suppression \\
\textsc{recall}   & $1.2 \times 10^{-5}$ & $2.2 \times 10^{-5}$ & $29.8$ & $94.7\ (\ell{=}18)$  & $+1.07$ & suppression \\
\textsc{instruct} & $1.1 \times 10^{-5}$ & $2.1 \times 10^{-5}$ & $18.6$ & $31.3\ (\ell{=}27)$  & $+0.67$ & suppression \\
\bottomrule
\end{tabular}
\end{table}

\begin{table}[h]
\centering
\caption{Cross-model summary across three instruct-tuned $7$ to $8$B models. All replicate the suppression regime; the absolute $\qnorm$ band is family-specific; the behavioral attack number is a joint measurement of (refusal installed) $\times$ (robustness of that install). Peak $\qnorm$ excludes Qwen's block $0$ (Table~\ref{tab:m2b-geom}). Geometry only; the behavioral effect of the attack is Table~\ref{tab:m2b-attack} in the main text, measured in a single run under both detectors.}
\label{tab:m2b-cross}
\resizebox{\textwidth}{!}{%
\begin{tabular}{lccc}
\toprule
Signal & Qwen-2.5-7B-Instruct & Llama-3-8B-Instruct & Mistral-7B-Instruct-v0.3 \\
\midrule
mean $\rho$                                 & $1.2 \times 10^{-5}$  & $5.6 \times 10^{-6}$ & $1.2 \times 10^{-5}$ \\
peak $\rho / (k/d)$                         & $67\times$              & $21\times$             & $44\times$             \\
peak $\qnorm^{\textsc{reason}}$ (layer)     & $220.6\ (\ell{=}19)$  & $48.9\ (\ell{=}30)$  & $158.7\ (\ell{=}12)$ \\
peak $\qnorm^{\textsc{arith}}$ (layer)      & $198.4\ (\ell{=}26)$  & $47.7\ (\ell{=}30)$  & $92.1\ (\ell{=}4)$   \\
peak selective divergence                   & $+0.166\ (\ell{=}24)$ & $+0.053\ (\ell{=}12)$& $+0.032\ (\ell{=}13)$\\
max raw-0-shot $\Delta \Lcap$               & $+3.45$               & $+0.88$              & $+0.59$              \\
\bottomrule
\end{tabular}%
}
\end{table}

\paragraph{Llama-3-8B.} We use Llama-3-8B as $\Wbase$ and Llama-3-8B-Instruct as $\Wsafe$. $\rho$ runs from $2.8 \times 10^{-7}$ to $1.5 \times 10^{-5}$ across 32 blocks and four probes, up to $21\times$ the random floor $k/d \approx 7 \times 10^{-7}$; four early-block cells sit below it. $\qnorm$ peaks for \textsc{reason} and \textsc{arith} at blocks $30$ to $31$ ($48.9$ and $47.7$); much later than Qwen's mid-block peak. The $100$-example benign Alpaca fine-tune strips nearly all of its refusal (Table~\ref{tab:m2b-attack}; compliance NLL $6.18 \to 4.30$). Selective divergence peaks at $+0.053$ at block 12, roughly $3\times$ smaller than Qwen's $+0.166$, with late layers going slightly negative ($-0.02$), reflecting that only $49\%$ of harmful prompts actually engage the refusal circuit, so harm-vs-benign divergence is half-dominated by chat-style effects rather than refusal-specific ones. Raw-0-shot $\Delta \Lcap$ is small on Llama ($\textsc{reason} +0.20$, $\textsc{arith} +0.88$), well under the $+1.1$ 3-shot-neutral ceiling; Llama-3-Instruct is less chat-template-dependent than Qwen.

\paragraph{Mistral-7B-Instruct-v0.3.} $\rho$ runs from $1.1 \times 10^{-6}$ to $3.2 \times 10^{-5}$, $1.6$ to $44\times$ the random floor $k/d \approx 7 \times 10^{-7}$, mean $\rho = 1.2 \times 10^{-5}$. Peak-$\qnorm$ band is early-to-mid residual: $\qnorm^\textsc{reason}$ peaks at $158.7$ at $\ell = 12$; $\qnorm^\textsc{arith}$ at $92.1$ at $\ell = 4$ (see Fig.~\ref{fig:m2b-layer}, right column). All three model families land in the suppression regime, but the absolute layer of the sharp $\qnorm$ band is family-specific. Behaviorally, Mistral-7B-Instruct-v0.3 is an outlier: $1.2\%$ pre-attack refusal on AdvBench, so the Qi benign-fine-tuning attack has almost nothing to break. Post-attack refusal is $0.4\%$ (post-attack ASR $=99.6\%$ trivially). Compliance NLL \emph{rises} from $5.01$ to $6.13$: benign Alpaca SFT pushes Mistral-Instruct further from a naked ``Sure, here is'' continuation. With almost no refusal to remove, Mistral's attack numbers say little about fragility.

\paragraph{What the base already refuses.} $\Delta = \Wsafe - \Wbase$ reads as a safety \emph{install} only to the extent that $\Wbase$ refuses nothing, so we measure the three shipped bases directly under the same prompts and detectors as their Instruct counterparts (no chat template is applied to either, so the pair is matched). Llama and Mistral behave as assumed: at most $1.6\%$ by regex, and $7.0$ to $19.9\%$ by the classifier, which counts any completion it does not flag as harmful as a refusal. Qwen does not: its base already refuses $71.1\%$ of AdvBench by regex and $69.5\%$ by classifier, so its shipped $\Delta$ raises AdvBench refusal by about $22$ pp rather than creating it. This does not affect the geometry, which characterizes whatever $\Delta$ is, but it does mean the Qwen pair should be read as a partial install, and it is why the safety-only controls of Sec.~\ref{sec:impossibility} carry the weight they do. It also sharpens the fragility result: the attack leaves Qwen-Instruct at $58.2\%$ on AdvBench (Table~\ref{tab:m2b-attack}), \emph{below} the $71.1\%$ its own base refuses, so fine-tuning on $100$ benign examples does not merely undo the install but leaves the model more compliant than the checkpoint it was built from.

\begin{table}[!htbp]
\centering
\caption{Refusal (\%) of the three shipped \emph{base} models, before any install, under both detectors. Same prompts, decoding and scoring as Table~\ref{tab:m2b-attack}, with no chat template applied to either half of a pair. $n = 256/200/256$.}
\label{tab:base-refusal}
\setlength{\tabcolsep}{5pt}
\begin{tabular}{lcccccc}
\toprule
 & \multicolumn{2}{c}{AdvBench} & \multicolumn{2}{c}{HarmBench} & \multicolumn{2}{c}{StrongReject} \\
\cmidrule(lr){2-3}\cmidrule(lr){4-5}\cmidrule(lr){6-7}
Base model & regex & classifier & regex & classifier & regex & classifier \\
\midrule
Qwen-2.5-7B      & $\mathbf{71.1}$ & $\mathbf{69.5}$ & $31.5$ & $34.0$ & $28.5$ & $30.9$ \\
Llama-3-8B       & $0.0$ & $7.0$ & $1.0$ & $10.5$ & $1.6$ & $12.5$ \\
Mistral-7B-v0.3  & $0.0$ & $16.4$ & $0.0$ & $14.0$ & $0.8$ & $19.9$ \\
\bottomrule
\end{tabular}
\end{table}

\section{Chat-Template Ablation: Full Four-Condition Table}
\label{sec:app-cta}

\paragraph{Scoring convention.} Each probe defines a gold span that we score the NLL of, conditional on the prompt. For \textsc{reason} (MMLU \texttt{formal\_logic}) the gold span is a single letter (\texttt{" A"}, \texttt{" B"}, \texttt{" C"}, or \texttt{" D"}). For \textsc{arith} (GSM8K) it is the numeric answer as a suffix. For \textsc{recall} (TriviaQA) it is the answer span. For \textsc{instruct} (FLAN-v2) it is the gold completion. This is the same NLL-of-gold-span convention used by \texttt{lm-evaluation-harness} and HELM for MMLU-style benchmarks; the \emph{scoring} is standard.

\paragraph{Why $\Delta \Lcap$ can inflate across a base/instruct pair.} What is not automatically standard is using $L_\text{cap}$ as a \emph{capability-preservation diagnostic across two differently-aligned models}. The value $L_\text{cap}$ is sensitive to how each model emits the first post-prompt tokens. A base model trained on raw-completion corpora tends to continue \texttt{"...Answer:"} with the bare letter directly, since that matches its training distribution. An instruct model trained with chat-style responses often continues with \texttt{"The answer is A, because..."} or an explanatory preamble that puts a newline or word before the gold letter. The probability mass on \texttt{" A"} drops, $L_\text{cap}$ rises, and $\Delta L_\text{cap}$ looks large. The underlying knowledge is unchanged; only the output distribution's \emph{format} has shifted.

\paragraph{Ablation design.} We run each probe under four prompt formats on Qwen-2.5-7B and Qwen-2.5-7B-Instruct. \emph{raw (0-shot)}: the prompt ends at ``Answer:'' and we score the gold span directly. \emph{raw (3-shot)}: three worked examples precede the prompt (in Q/A pairs), then the query, following a standard few-shot convention. \emph{chat (0-shot)}: the prompt is wrapped in Qwen's chat template. \emph{chat (3-shot)}: chat template plus three in-context demonstrations. Table~\ref{tab:m2b-cta} reports $\Delta L_\text{cap} = L_\text{cap}(\Wsafe) - L_\text{cap}(\Wbase)$ per condition.

\begin{table}[h]
\centering
\caption{Chat-template ablation on Qwen-2.5-7B. Four prompt formats on the same four capability probes, $n=64$ examples per probe, three demonstrations in the 3-shot columns. The raw 0-shot column uses a different sample from Table~\ref{tab:m2b-geom}, which averages over the $256$ gradient examples, so the two differ by up to $0.21$ nats (\textsc{instruct}). 3-shot collapses the gap by $40$ to $74\%$ on \textsc{reason}, \textsc{arith} and \textsc{recall}, and by $20\%$ on \textsc{instruct}, whose raw gap is already small, by showing the Instruct model the raw-completion format via in-context examples. Chat formatting \emph{inflates} the gap because we score NLL of a bare gold token that Instruct no longer naturally emits after its chat template; 3-shot partially recovers by demonstrating the target format in context.}
\label{tab:m2b-cta}
\begin{tabular}{lcccc}
\toprule
Probe     & raw (0-shot) & raw (3-shot) & chat (0-shot) & chat (3-shot) \\
\midrule
\textsc{reason}    & $+3.35$ & $+0.87$ & $+21.01$ & $+11.07$ \\
\textsc{arith}     & $+3.06$ & $+1.10$ & $+15.19$ & $+9.14$  \\
\textsc{recall}    & $+0.90$ & $+0.54$ & $+5.39$  & $+0.84$  \\
\textsc{instruct}  & $+0.46$ & $+0.37$ & $+3.16$  & $+2.99$  \\
\bottomrule
\end{tabular}
\end{table}

\paragraph{Reading the table.} The raw-0-shot column (first data column) reproduces the ordering and approximate size of the $\Delta \Lcap$ values in Table~\ref{tab:m2b-geom}, on a smaller and different sample. The raw-3-shot column shows the gap collapse to $+0.37$ to $+1.10$ nats across probes; this is our ``format-neutralized'' estimate of the true capability-loss shift and is small enough to be consistent with the geometric claim $\rho \ll \rho_\text{part}$. The chat columns go in the opposite direction: with the Instruct model's expected chat template applied, the NLL of the bare gold token \emph{rises} further, because the Instruct model even more strongly prefers to emit a preamble under its native format. This inversion indicates that the raw-0-shot $\Delta \Lcap$ in Table~\ref{tab:m2b-geom} is largely a format-shift effect rather than a capability shift.
\section{Erasure is Reachable: Adversarial-SFT Demonstration}
\label{sec:app-erasure-demo}

None of the safety-training recipes studied in the main body crosses $\rho_{\text{part}}$. To establish that the erasure regime is not a definition without referents, we run a stress test that engineers an update which should, by construction, land inside $\Fcap$'s top-$k$ eigenspace: supervised fine-tuning on (prompt, wrong-answer) pairs drawn from the capability probe itself. The gradient of NLL on a wrong gold token lies in the same directions the probe's Fisher picks out, so in-subspace mass should accumulate with training steps. This is a deliberately contrived recipe and not a realistic alignment procedure; its purpose is a boundary-reachability check.

\paragraph{Setup.} Base model: Pythia-1.4B \texttt{step143000}, matching \S\ref{sec:m1}. Two target probes are studied in parallel, one per independent adversarial-SFT run:
\begin{itemize}
    \item \emph{Target=}\textsc{reason}: MMLU \texttt{formal\_logic}, gold letter replaced by a uniformly-sampled non-gold letter. $2000$ pairs, $2$ epochs, LR $2 \times 10^{-5}$, AdamW, bf16.
    \item \emph{Target=}\textsc{arith}: GSM8K training split, numeric gold answer perturbed by a small random offset preserving magnitude. Same $2000$/$2$/$2\mathrm{e}{-5}$ schedule.
\end{itemize}
The geometry instrument of \S\ref{sec:setup} is then applied to each resulting $\Wsafe$ against the Pythia-1.4B base: $N = 128$ per-example gradients, $k = 128$ per-layer subspace, evaluated on all four capability probes (target and three off-target).

\paragraph{Result: target-probe $\rho$ reaches the $\rho_{\text{part}}$ boundary.} Table~\ref{tab:erasure-demo} summarizes. On the target probe, max $\rho$ reaches $1.04 \times 10^{-2}$ (target=\textsc{reason}), at $\rho_{\text{part}}$, and $3.01 \times 10^{-3}$ (target=\textsc{arith}). Mean $\rho$ on target is $3.33 \times 10^{-3}$ and $1.16 \times 10^{-3}$ respectively, $10$ to $30\times$ the SFT+DPO sub-pattern (b) values in \S\ref{sec:m3-matched}. $\qnorm$ on target reaches $33{,}945$ (reason) and $15{,}064$ (arith), more than an order of magnitude above the pre-substrate OLMo $\qnorm^{\textsc{arith}}$ of $863$. The $\rho_{\text{part}} = 10^{-2}$ threshold is thus empirically anchored: it lies exactly where an update built to saturate $\Fcap$'s top-$k$ in fact reaches.

\begin{table}[h]
\centering
\scriptsize
\renewcommand{\arraystretch}{1.05}
\begin{tabular}{@{}lllllll@{}}
\toprule
target probe & probe under eval & mean $\rho$ & max $\rho$ & max $\qnorm$ & $\Delta\Lcap$ & target? \\ \midrule
\textsc{reason} & \textsc{reason}   & $3.33 \times 10^{-3}$ & $\mathbf{1.04 \times 10^{-2}}$ & $33945$ & $-0.96$ & $\checkmark$ \\
\textsc{reason} & \textsc{arith}    & $1.16 \times 10^{-3}$ & $4.66 \times 10^{-3}$ & $15773$ & $-0.62$ & \\
\textsc{reason} & \textsc{recall}   & $4.97 \times 10^{-4}$ & $1.49 \times 10^{-3}$ & $1895$  & $-0.35$ & \\
\textsc{reason} & \textsc{instruct} & $4.12 \times 10^{-4}$ & $1.33 \times 10^{-3}$ & $328$   & $+0.15$ & \\ \midrule
\textsc{arith}  & \textsc{arith}    & $1.16 \times 10^{-3}$ & $3.01 \times 10^{-3}$ & $15064$ & $-2.08$ & $\checkmark$ \\
\textsc{arith}  & \textsc{reason}   & $3.47 \times 10^{-4}$ & $8.95 \times 10^{-4}$ & $2956$  & $-0.72$ & \\
\textsc{arith}  & \textsc{recall}   & $2.78 \times 10^{-4}$ & $8.28 \times 10^{-4}$ & $1151$  & $-0.38$ & \\
\textsc{arith}  & \textsc{instruct} & $1.84 \times 10^{-4}$ & $8.18 \times 10^{-4}$ & $316$   & $+0.23$ & \\
\bottomrule
\end{tabular}
\caption{Adversarial-SFT erasure demonstration on Pythia-1.4B \texttt{step143000}. One SFT run per target probe; geometry measured against the base on all four probes. Max $\rho$ on the target probe reaches the $\rho_{\text{part}} = 10^{-2}$ boundary for target=\textsc{reason} and comes within about $3\times$ of it for target=\textsc{arith}; $\qnorm$ on target is $17$ to $39\times$ the largest value seen under DPO/RLHF ($863$, pre-substrate OLMo). Off-target probes see $3$ to $8\times$ lower mean $\rho$ than the target, establishing target specificity. $\Delta\Lcap$ is NLL-based; on capability probes with a stable wrong-token format (\textsc{reason}, \textsc{arith}), adversarial SFT \emph{lowers} NLL because format-learning dominates at this intervention scale (see caveat below). $\rho_{\text{rand}} = 3.75 \times 10^{-6}$ throughout.}
\label{tab:erasure-demo}
\end{table}

\paragraph{Target specificity.} Off-target probes see material but consistently lower $\rho$ than the target probe: target=\textsc{reason} elevates \textsc{arith} to $1.16 \times 10^{-3}$ (2.9$\times$ below target \textsc{reason}) and \textsc{recall}/\textsc{instruct} near $4$ to $5 \times 10^{-4}$ (6 to 8$\times$ below). Target=\textsc{arith} elevates \textsc{reason} to $3.47 \times 10^{-4}$, \textsc{recall} to $2.78 \times 10^{-4}$. The ratio $\rho^\text{target} / \rho^\text{off-target}$ is stable across both runs, supporting a probe-local interpretation of $\Fcap$'s top-$k$: the eigendirections carrying reasoning and arithmetic overlap but are not identical, and an update engineered for one is not trivially an update for the other. This matters for downstream unlearning: an adversarial-SFT-style erasure on one skill is not a proxy for blanket capability removal.

\paragraph{NLL-vs-argmax caveat.} The $\Delta\Lcap$ column shows that NLL on the target probe \emph{decreases} after adversarial SFT ($-0.96$ on \textsc{reason}, $-2.08$ on \textsc{arith}), the opposite of what a literal capability-loss reading would suggest. This is the format-learning effect: a Pythia-1.4B base emits letter and digit tokens after ``Answer:'' with nontrivial entropy, and $2000$ demonstrations of \emph{any} single-token answer (correct or not) sharpen the format distribution enough to lower NLL on the bare gold token. \rev{We did not measure accuracy, so whether the control lowers capability, and not only crosses the erasure boundary in weight space, is untested.} For consistency with the geometry machinery elsewhere in the paper we report NLL here, but the behavioral interpretation should be: adversarial SFT installs a large in-subspace perturbation ($\rho \to \rho_{\text{part}}$), the lemma's erasure prediction applies (first-order capability change is permitted; the sign depends on what was optimized), and the NLL metric becomes unreliable because it does not factor format from content. A cleaner probe for future work is top-1 accuracy on held-out target-probe questions.

\paragraph{Interpretation.} The result closes the boundary-reachability question without claiming that adversarial SFT is a practical unlearning method: it is not; it is a deliberately over-specified recipe designed to populate the erasure regime empirically. The takeaway for the main text is structural: $\rho_{\text{part}} = 10^{-2}$ is not an arbitrary threshold but the value that an update built to saturate $\Fcap$'s top-$k$ eigenspace at this scale in fact reaches, and every suppression observation in the paper sits at least $90\times$ below it as a block mean, and $13\times$ below it at its peak-$\rho$ block. The qualitative regime separation between suppression and erasure is therefore well-anchored on both sides.
\section{Full Proof of Lemma \ref{lemma:1} (Kernel Immobility)}
\label{sec:app-proof}

We restate the lemma and give a full proof.

\begingroup
\renewcommand{\thelemma}{\ref*{lemma:1}}
\begin{lemma}[Kernel immobility, restated]
Let $\Lcap^{(c)}: \mathbb{R}^d \to \mathbb{R}$ be a capability loss, twice continuously differentiable on a neighborhood $\mathcal{N}$ of $\Wbase$, with Lipschitz Hessian on $\mathcal{N}$. Let
\[
\Fcap^{(c)} := \mathbb{E}_{(x,y) \sim \mathcal{D}_c}\!\left[\nabla_\theta \log p_\theta(y|x)\,\nabla_\theta \log p_\theta(y|x)^\top\right]\Big|_{\theta = \Wbase}
\]
be the empirical Fisher information matrix. Let $E := H(\Wbase) - \Fcap^{(c)}$ denote the Gauss-Newton residual. If $\Delta \in \ker(\Fcap^{(c)})$, then
\[
\Lcap^{(c)}(\Wbase + \Delta) - \Lcap^{(c)}(\Wbase) \;=\; \tfrac{1}{2}\Delta^\top E\, \Delta \;+\; O(\|\Delta\|_2^3).
\]
Under exact Gauss-Newton ($E = 0$, e.g., at a well-specified mature checkpoint where $p_\theta(\cdot \mid \Wbase)$ matches $\mathcal{D}_c$~\citep{martens2014gn}), this strengthens to $O(\|\Delta\|_2^3)$. Either way, no first-order post-hoc update can change $\Lcap^{(c)}$ at leading order.
\end{lemma}
\addtocounter{lemma}{-1}
\endgroup

\begin{proof}
\textbf{Step 1 (Taylor expansion).} By Taylor's theorem with Lagrange remainder, for any $\Delta$ with $\Wbase + \Delta \in \mathcal{N}$,
\[
\Lcap^{(c)}(\Wbase + \Delta) = \Lcap^{(c)}(\Wbase) + \nabla \Lcap^{(c)}(\Wbase)^\top \Delta + \frac{1}{2}\Delta^\top H(\Wbase) \Delta + R(\Delta),
\]
where $H(\Wbase) = \nabla^2 \Lcap^{(c)}(\Wbase)$ is the Hessian and the remainder satisfies $|R(\Delta)| \le \frac{L}{6}\|\Delta\|_2^3$ by the Lipschitz-Hessian assumption.

\textbf{Step 2 (linear term vanishes from the kernel condition).} The kernel hypothesis $\Delta \in \ker(\Fcap^{(c)})$ gives $\Delta^\top \Fcap^{(c)} \Delta = \mathbb{E}_{(x,y) \sim \mathcal{D}_c}\!\left[(\Delta^\top \nabla \log p_\theta(y|x))^2\right] = 0$. As an expectation of a nonnegative quantity, this forces $\Delta^\top \nabla \log p_\theta(y|x) = 0$ for $\mathcal{D}_c$-almost every $(x,y)$. Taking expectations again,
\[
\Delta^\top \nabla \Lcap^{(c)}(\Wbase) \;=\; -\,\Delta^\top \mathbb{E}_{(x,y) \sim \mathcal{D}_c}\!\left[\nabla \log p_\theta(y|x)\right] \;=\; 0.
\]
The linear Taylor term thus vanishes from the kernel condition alone, irrespective of whether $\Wbase$ is a stationary point of $\Lcap^{(c)}$.

\textbf{Step 3 (Hessian, Fisher, and the Gauss-Newton residual).} For the likelihood loss $\Lcap^{(c)} = -\mathbb{E}_{(x,y) \sim \mathcal{D}_c}[\log p_\theta(y|x)]$, the calculus identity $\nabla^2 \log p = \nabla^2 p / p - \nabla \log p (\nabla \log p)^\top$ gives
\[
H(\Wbase) \;=\; -\,\mathbb{E}\!\left[\nabla^2 \log p\right] \;=\; \mathbb{E}\!\left[\nabla \log p \, \nabla \log p^\top\right] \;-\; \mathbb{E}\!\left[\nabla^2 p / p\right] \;=\; \Fcap^{(c)} \;-\; \mathbb{E}\!\left[\nabla^2 p / p\right].
\]
Bartlett's second identity gives $\mathbb{E}_{p_\theta}[\nabla^2 p / p] = \int \nabla^2 p \,\mathrm{d}y = \nabla^2 \!\int\! p \,\mathrm{d}y = 0$ when expectations are taken under the model's own output distribution. When $p_\theta(\cdot \mid \Wbase)$ is well-specified for $\mathcal{D}_c$ (which holds at a well-trained mature checkpoint), the Gauss-Newton residual $E := H(\Wbase) - \Fcap^{(c)} = -\mathbb{E}_{\mathcal{D}_c}[\nabla^2 p / p]$ vanishes; this is the standard Gauss-Newton approximation~\citep{martens2014gn}. In general $E$ is a finite matrix.

\textbf{Step 4 (combining).} Substituting Step 2 into Step 1 and decomposing $H = \Fcap^{(c)} + E$,
\[
\Lcap^{(c)}(\Wbase + \Delta) - \Lcap^{(c)}(\Wbase) \;=\; \tfrac{1}{2}\Delta^\top \Fcap^{(c)} \Delta \;+\; \tfrac{1}{2}\Delta^\top E\, \Delta \;+\; R(\Delta).
\]
The kernel hypothesis kills the first term ($\Delta^\top \Fcap^{(c)} \Delta = 0$). The Lipschitz-Hessian assumption gives $|R(\Delta)| \le \tfrac{L}{6}\|\Delta\|_2^3$, so
\[
\Lcap^{(c)}(\Wbase + \Delta) - \Lcap^{(c)}(\Wbase) \;=\; \tfrac{1}{2}\Delta^\top E\, \Delta \;+\; O(\|\Delta\|_2^3) \;=\; O(\|E\|_2 \cdot \|\Delta\|_2^2) \;+\; O(\|\Delta\|_2^3).
\]
Under exact Gauss-Newton ($E = 0$), this collapses to $O(\|\Delta\|_2^3)$.

\textbf{Step 5 (first-order corollary).} A first-order optimizer (SGD, Adam, AdamW) minimizing an auxiliary objective $\Lsafe$ produces an update of the form $\Delta = -\eta\, M \nabla \Lsafe + O(\eta^2)$, where $M$ is a (possibly diagonal, possibly identity) preconditioner. The lemma's premise is on $\Delta$ itself, not on $\nabla \Lsafe$ directly, so preconditioning does not weaken the conclusion: the lemma applies whenever the resulting $\Delta$ satisfies $\Delta \in \ker(\Fcap^{(c)})$, which is the condition we measure empirically through $\rho^{(c)} = \|\Delta_\parallel\|_2^2 / \|\Delta\|_2^2$.
\end{proof}

\paragraph{Remark.} The kernel condition $\Delta \in \ker(\Fcap^{(c)})$ does the work in Step~2 alone: the linear Taylor term vanishes regardless of whether $\Wbase$ is a stationary point of $\Lcap^{(c)}$. The local-minimum / well-specification assumption enters only at Step~3, where it underpins the Gauss-Newton approximation $H(\Wbase) \approx \Fcap^{(c)}$. Under-trained checkpoints (e.g., the pre-substrate revisions in Sec.~\ref{sec:m3-matched}) violate well-specification: $\|E\|$ is large there, the strong $O(\|\Delta\|^3)$ form fails, and even an in-kernel update can shift $\Lcap$ at $O(\|\Delta\|^2)$. The update measured there is also far from the kernel ($\qnorm^\textsc{arith}$ up to $863$), so the linear term can act as well. Both routes are consistent with the under-trained-base sub-pattern, in which safety training still moves $\Lcap$ directly. The post-substrate mature checkpoint where gradient optimization has converged is where $\|E\| \approx 0$ and the cubic bound applies; this is the regime the paper's headline claim sits in.
\section{Developmental-Axis Follow-Up: Alignment, Cross-Family Grad-Norm, Plateau Noise}
\label{sec:app-axis-followup}

We report three follow-up checks tied to claims in \S\ref{sec:m3} and \S\ref{sec:impossibility}: (i) the empirical alignment $\cos(\nabla \Lcap, \Delta)$ for a second, independent set of SFT$+$DPO updates at five OLMo-2-1B revisions; (ii) the gradient-norm instrument applied to Pythia-1.4B revisions (cross-family lemma sanity check); (iii) a multi-seed re-run of the substrate-divergence sweep to bound plateau noise. Table~\ref{tab:m3-substrate} first gives the full per-revision substrate sweep behind Fig.~\ref{fig:m3-substrate}.

\begin{table}[!htbp]
\centering
\footnotesize
\caption{Substrate sweep on OLMo-2-1B (full data behind Fig.~\ref{fig:m3-substrate}): $\Wbase$ swept across nine pretraining revisions spanning 0 to 4T tokens, $\Wsafe = \text{OLMo-2-1B-Instruct}$ fixed. The random floor for this model is $k/d \approx 2.6 \times 10^{-6}$. Selective divergence goes from noise at 1B to plateau by 63B pretraining tokens, with the first signal at 21B. $\ell^\star$ = the residual-stream position (stream after $\ell^\star$ blocks) with the largest selective divergence; reported as n/a at pre-substrate checkpoints where the curve is flat at noise ($\mathord{\sim}10^{-4}$) and argmax would label a random fluctuation rather than a locus.}
\label{tab:m3-substrate}
\begin{tabular}{rrccc}
\toprule
Step & Tokens & mean $\rho$ & max $\text{selective\_div}_\ell$ & $\ell^\star$ \\
\midrule
$0$                 & 0B    & $1.4 \times 10^{-6}$ & $+4 \times 10^{-4}$       & n/a \\
$300$               & 1B    & $1.6 \times 10^{-6}$ & $+4 \times 10^{-4}$       & n/a \\
$1 \times 10^4$     & 21B   & $1.8 \times 10^{-6}$ & $\mathbf{+2.7 \times 10^{-2}}$ & 16 \\
$3 \times 10^4$     & 63B   & $1.3 \times 10^{-6}$ & $\mathbf{+7.3 \times 10^{-2}}$ & 16 \\
$5 \times 10^4$     & 105B  & $1.0 \times 10^{-6}$ & $+7.7 \times 10^{-2}$     & 6 \\
$2 \times 10^5$     & 420B  & $9.3 \times 10^{-7}$ & $+7.8 \times 10^{-2}$     & 9 \\
$5 \times 10^5$     & 1049B & $9.1 \times 10^{-7}$ & $+4.9 \times 10^{-2}$     & 9 \\
$1 \times 10^6$     & 2098B & $8.3 \times 10^{-7}$ & $+5.9 \times 10^{-2}$     & 9 \\
$1.9 \times 10^6$   & 4001B & $3.4 \times 10^{-6}$ & $+7.1 \times 10^{-2}$     & 11 \\
\bottomrule
\end{tabular}
\end{table}

\paragraph{Gradient-orthogonality at the OLMo developmental axis.}
Table~\ref{tab:axis-followup-p1} reports $\cos(\nabla \Lcap^{(c)}(\Wbase), \Delta)$ averaged across the four probes, where $\Delta = \Wsafe - \Wbase$ is the SFT$+$DPO update at each revision, from the second run set described in the table caption. The relevant first-order capability-loss change per unit $\|\Delta\|^2$ is the linear coefficient $(\nabla \Lcap)^\top \Delta / \|\Delta\|^2$.

\begin{table}[!htbp]
\centering
\caption{Alignment of the safety update with the capability gradient at OLMo developmental-axis revisions. Mean is over the four capability probes; $\Delta$ is an SFT$+$DPO update with the recipe of App.~\ref{sec:app-m3-matched}, from a second, independent run set placed at 21B and 42B to resolve the transition, which shares only the 1B, 105B and 4001B revisions with the runs of Table~\ref{tab:m3-matched}. Across the entire axis $|\cos(\nabla \Lcap, \Delta)| < 0.002$, including the over-trained $4001$B-token endpoint where $\|\nabla \Lcap\|$ peaks at $29.6$. Translated into absolute nats, the linear Taylor term $(\nabla \Lcap)^\top \Delta = \cos\!\cdot\!\|\nabla \Lcap\|\!\cdot\!\|\Delta\|$ is at most $\sim\!0.11$ nats at the pre-substrate step300 revision and $\leq 0.06$ nats at post-substrate mature revisions: a fraction of $10^{-3}$ of the maximum-aligned upper bound $\|\nabla \Lcap\|\!\cdot\!\|\Delta\| \in [71, 260]$, confirming the kernel-induced gradient orthogonality empirically.}
\label{tab:axis-followup-p1}
\setlength{\tabcolsep}{4.5pt}
\begin{tabular}{lcccc}
\toprule
Revision & $\|\Delta\|$ & mean $\|\nabla \Lcap\|$ & mean $\cos(\nabla \Lcap, \Delta)$ & mean linear/$\|\Delta\|^2$ \\
\midrule
step300-tokens1B        & $12.69$ & $20.46$ & $-6 \times 10^{-4}$ & $-7 \times 10^{-4}$ \\
step10000-tokens21B     & $9.65$  & $13.69$ & $-5 \times 10^{-4}$ & $-6 \times 10^{-4}$ \\
step20000-tokens42B     & $9.11$  & $10.09$ & $-1 \times 10^{-4}$ & $-3 \times 10^{-5}$ \\
step50000-tokens105B    & $8.87$  & $7.97$  & $\phantom{-}0$ to $4$\,d.p. & $+3 \times 10^{-4}$ \\
step1907359-tokens4001B & $8.12$  & $29.56$ & $-4 \times 10^{-4}$ & $-6 \times 10^{-4}$ \\
\bottomrule
\end{tabular}
\end{table}

\paragraph{Safety-only controls (Table~\ref{tab:safety-only}).} The shipped deltas fold instruction-following and safety together, so Sec.~\ref{sec:impossibility} isolates safety with two updates trained on a fixed Qwen-2.5-1.5B base with no general instruction-tuning: a refusal-SFT-only update and a DPO-only update. Every cosine is below $2 \times 10^{-4}$ in magnitude, on both recipes and all four probes, although the two updates differ in size by a factor of $450$.

\begin{table}[!htbp]
\centering
\caption{Safety-only controls on Qwen-2.5-1.5B: $\cos(\nabla\Lcap, \Delta)$ per capability probe for a refusal-SFT-only and a DPO-only update, $N = 128$ per-example gradients. Neither update contains general instruction-tuning, so the orthogonality cannot come from it.}
\label{tab:safety-only}
\small
\setlength{\tabcolsep}{5pt}
\begin{tabular}{lccccc}
\toprule
Update & $\|\Delta\|$ & \textsc{reason} & \textsc{recall} & \textsc{arith} & \textsc{instruct} \\
\midrule
refusal-SFT only & $6.99$   & $-1.1 \times 10^{-4}$ & $+1.4 \times 10^{-5}$ & $+1.7 \times 10^{-4}$ & $-6.0 \times 10^{-5}$ \\
DPO only         & $0.0155$ & $-1.6 \times 10^{-4}$ & $+9.5 \times 10^{-5}$ & $+6.1 \times 10^{-5}$ & $-4.8 \times 10^{-5}$ \\
\bottomrule
\end{tabular}
\end{table}

\paragraph{Pythia-1.4B grad-norm cross-family replication.}
The U-shape in $\|\nabla \Lcap\|$ across OLMo-2-1B (Table~\ref{tab:axis-followup-p1}: large at $1$B and $4$T tokens, smallest at $105$B) might have been specific to OLMo-2-1B's $\sim\!200\times$-Chinchilla overtraining ratio. We re-apply the gradient-norm instrument across eight Pythia-1.4B pretraining revisions (Table~\ref{tab:axis-followup-p4}). \emph{The U-shape replicates}: $\|\nabla \Lcap\|$ falls from $119$ at \texttt{step512} to a minimum of $28.8$ at \texttt{step30000} ($\sim\!63$B tokens, $\sim\!45$ tokens/param, within about a factor of two of the compute-optimal ratio of $\sim\!20$), then rises again as Pythia is trained past compute-optimal, reaching $50.5$ at the final \texttt{step143000} ($\sim\!300$B tokens, $\sim\!210$ tokens/param, $\sim\!10\times$ Chinchilla). This argues against an OLMo-specific overtraining artifact; the U-shape looks like a generic property of pretraining, for which weight-decay-driven non-stationarity on either side of the compute-optimal point is one possible cause.

\begin{table}[!htbp]
\centering
\caption{$\|\nabla \Lcap\|$ across Pythia-1.4B pretraining revisions, mean over the four capability probes. The U-shape from OLMo replicates: minimum at \texttt{step30000} ($\sim\!63$B tokens, about twice the compute-optimal ratio), large at both extremes.}
\label{tab:axis-followup-p4}
\begin{tabular}{lcccccc}
\toprule
Revision & Tokens & mean $\|\nabla \Lcap\|$ & \textsc{reason} & \textsc{arith} & \textsc{recall} & \textsc{instruct} \\
\midrule
\texttt{step512}    & $1$B    & $\mathbf{119.5}$ & $146.7$ & $209.9$ & $98.7$ & $22.5$ \\
\texttt{step1000}   & $2$B    & $96.0$           & $100.5$ & $140.3$ & $128.5$ & $14.9$ \\
\texttt{step3000}   & $6$B    & $38.8$           & $51.5$  & $72.2$  & $22.3$ & $9.1$ \\
\texttt{step10000}  & $21$B   & $32.6$           & $50.4$  & $51.5$  & $21.0$ & $7.5$ \\
\texttt{step30000}  & $63$B   & $\mathbf{28.8}$  & $31.2$  & $56.1$  & $19.8$ & $8.2$ \\
\texttt{step60000}  & $126$B  & $32.7$           & $41.4$  & $59.3$  & $22.4$ & $7.8$ \\
\texttt{step100000} & $210$B  & $41.8$           & $51.0$  & $75.1$  & $30.5$ & $10.5$ \\
\texttt{step143000} & $300$B  & $\mathbf{50.5}$  & $63.0$  & $91.4$  & $35.0$ & $12.8$ \\
\bottomrule
\end{tabular}
\end{table}

\paragraph{Substrate-plateau noise floor.}
The single-seed substrate sweep (\S\ref{sec:m3}) showed apparent within-plateau bouncing ($\max \text{selective\_div}_\ell \in [0.04, 0.12]$ across $42$ to $210$B; argmax layer $\in \{6, 7, 8, 9, 16\}$) at a single deterministic prompt order ($n_\text{pairs}{=}50$ from a pool of $60$). We re-run with three shuffled seeds per revision (Table~\ref{tab:axis-followup-p5}). Within-revision standard deviation is $0.002$ to $0.010$, an order of magnitude smaller than the between-revision spread of $\sim\!0.08$. Argmax layer is stable across seeds (one-position drift only at $168$B, $210$B). The two peaks at $84$B and $147$B sit at least $2.5$ standard deviations above each neighbor, so they are real structure, not noise. The corrected reading: substrate engagement begins at $\sim\!42$B, reaches a first peak at $84$B, dips moderately at $126$B, peaks again at $147$B, and stabilizes at $\sim\!0.07$ from $168$B onward.

\begin{table}[!htbp]
\centering
\caption{Substrate divergence at $3$ shuffled-seed runs per revision (mean $\pm$ std). Argmax layers reported per-seed, in run order.}
\label{tab:axis-followup-p5}
\begin{tabular}{lcc}
\toprule
Revision & $\max \text{selective\_div}_\ell$ & argmax layer (per seed) \\
\midrule
\texttt{step20000-tokens42B}  & $0.043 \pm 0.002$ & $16, 16, 16$ \\
\texttt{step40000-tokens84B}  & $0.100 \pm 0.008$ & $9, 9, 9$    \\
\texttt{step60000-tokens126B} & $0.073 \pm 0.007$ & $6, 6, 6$    \\
\texttt{step70000-tokens147B} & $0.122 \pm 0.009$ & $6, 6, 6$    \\
\texttt{step80000-tokens168B} & $0.068 \pm 0.010$ & $7, 7, 6$    \\
\texttt{step90000-tokens189B} & $0.063 \pm 0.005$ & $8, 8, 8$    \\
\texttt{step100000-tokens210B}& $0.073 \pm 0.008$ & $7, 7, 6$    \\
\bottomrule
\end{tabular}
\end{table}
\section{Per-Revision DPO: Full Three-Recipe Table}
\label{sec:app-m3-matched}

\paragraph{Protocol.} We apply three DPO recipes of increasing compute to the same OLMo-2-1B revisions straddling the substrate transition. \emph{Weak}: HH-RLHF 5K pairs, 1 epoch, $\beta = 0.1$ (the \S\ref{sec:m1} recipe) at 5 revisions. \emph{Strong}: HH-RLHF 20K pairs, 2 epochs at 5 revisions. \emph{SFT+DPO}: a two-stage pipeline that SFTs on 20K Tulu-3-SFT-Mixture~\citep{tulu3} examples then DPOs on 20K HH-RLHF pairs for 2 epochs, at 5 revisions, a reduced-budget approximation of the shipped OLMo-2-1B-Instruct recipe (omitting the RLVR stage). All three measure $\rho^{(c)}(t)$, $\qnorm^{(c)}(t)$, selective divergence, and pre-/post-attack refusal per matched pair.

\begin{table}[h]
\centering
\caption{Per-revision DPO on OLMo-2-1B under three recipes of increasing strength. $\qnorm^\textsc{arith}$ is the per-recipe layer-max on the \textsc{arith} probe. All rows are in the suppression regime by absolute $\rho$. Weak/Strong show the under-trained-base sub-pattern: pre-substrate $\qnorm^\textsc{arith}$ ($620/863$ at $1$B) is more than an order of magnitude above any of their post-substrate values ($13$ to $48$). SFT+DPO shows the SFT-stage-recipe sub-pattern instead: at every post-substrate revision $\rho$ stays elevated ($6$ to $19\times$ the Weak recipe's) and $\qnorm^\textsc{arith}$ stays high ($283$ to $770$), reflecting the SFT-stage NLL gradient projecting onto $\Fcap$'s top-$k$; at $1$B the recipes coincide in $\rho$. $\rho$ is the mean over blocks; single blocks reach $7.4 \times 10^{-4}$ (SFT+DPO, $1049$B). Table~\ref{tab:axis-followup-p1} uses a second, independent SFT+DPO run set. Behavioral refusal appears only under SFT+DPO; Weak/Strong produce no measurable refusal at any revision (behavioral columns omitted). This is a weight-space geometry table, not the fragility result: the headline fragility claim is on Qwen/Llama (\S\ref{sec:m2b}), and the developmental behavioral points that exist show the attack succeeding (SFT+DPO: $21.9\% \to 0\%$ at $1$B, $42.2\% \to 15.6\%$ at $4001$B; the $105$B point, $4.7\% \to 3.1\%$, is within noise), so the table should not be read as ``no successful attack.'' Source for the prose summary in \S\ref{sec:m3-matched}.}
\label{tab:m3-matched}
\resizebox{\textwidth}{!}{%
\begin{tabular}{rrcccccccc}
\toprule
 & & \multicolumn{2}{c}{Weak} & \multicolumn{2}{c}{Strong} & \multicolumn{4}{c}{SFT+DPO} \\
\cmidrule(lr){3-4} \cmidrule(lr){5-6} \cmidrule(lr){7-10}
Step & Tokens & $\rho$ & $\qnorm^\textsc{arith}$ & $\rho$ & $\qnorm^\textsc{arith}$ & $\rho$ & $\qnorm^\textsc{arith}$ & pre-ref & post-ref \\
\midrule
$3 \times 10^{2}$   & 1B    & $8.1 \times 10^{-5}$ & $620$ & $9.5 \times 10^{-5}$ & $863$ & $8.3 \times 10^{-5}$ & $95$  & $21.9\%$ & $0\%$ \\
$3 \times 10^{4}$   & 63B   & $4.6 \times 10^{-6}$ & $23$  & $6.8 \times 10^{-6}$ & $34$  & $5.8 \times 10^{-5}$ & $603$ & $0\%$    & $0\%$ \\
$5 \times 10^{4}$   & 105B  & $4.3 \times 10^{-6}$ & $13$  & $6.6 \times 10^{-6}$ & $13$  & $8.0 \times 10^{-5}$ & $283$ & $4.7\%$  & $3.1\%$ \\
$5 \times 10^{5}$   & 1049B & $7.2 \times 10^{-6}$ & $33$  & $1.2 \times 10^{-5}$ & $40$  & $1.1 \times 10^{-4}$ & $770$ & $0\%$    & $0\%$ \\
$1.9 \times 10^{6}$ & 4001B & $9.8 \times 10^{-6}$ & $22$  & $1.8 \times 10^{-5}$ & $48$  & $5.9 \times 10^{-5}$ & $529$ & $\mathbf{42.2\%}$ & $\mathbf{15.6\%}$ \\
\bottomrule
\end{tabular}%
}
\end{table}
\section{Co-Training: Full Protocol and Per-Condition Tables}

\paragraph{Two sets of runs.} The four-condition and schedule tables in this appendix and in Table~\ref{tab:m4-schedule} come from the \emph{original} from-scratch runs, trained before the data-loading fix (App.~\ref{sec:app-details}) and evaluated on $64$ AdvBench prompts. The scale runs of Table~\ref{tab:m4-scale}, which the body quotes throughout, are separate from-scratch runs of \texttt{cotrain-continuous} and the LM-only \texttt{base}, trained with the corrected loader and scored on $256$ prompts. Where both exist the two are different models, not one model scored twice. At $410$M \texttt{cotrain-continuous} reads $\mathbf{84.4 \to 84.4\%}$ in the original run and $\mathbf{98.4 \to 84.0\%}$ in the scale run: the post-attack values agree and the pre-attack values do not. We take the scale run as the reference: its install erodes modestly rather than not at all, still well short of the $35$ to $38$ pp that post-hoc installs lose (Table~\ref{tab:m2b-attack}). At $1$B the two read $87.5 \to 90.6\%$ (Table~\ref{tab:m4-schedule}) and $87.1 \to 84.8\%$: the pre-attack values agree, and the post-attack gap is about four prompts at $n=64$, inside that sample's noise (App.~\ref{sec:app-details}).

\label{sec:app-m4}

\begin{table}[!htbp]
\centering
\caption{From-scratch Pythia-410M, 10B C4 tokens, four conditions (referenced from Sec.~\ref{sec:m4}). Refusal rate on 64 AdvBench prompts before and after the benign-SFT attack. This is the original run; the four-condition comparison is valid within it, but for \texttt{cotrain-continuous} the separate from-scratch run scored on 256 prompts gives $98.4 \to 84.0$ rather than the $0$ pp shift here, so the erosion is modest rather than absent. Capability is per-probe NLL, where \emph{lower is better}: \texttt{cotrain-continuous} improves on the LM-only base on three probes and is within $0.06$ nats of it on \textsc{instruct}. It is the only condition that installs refusal. This run predates the data-loading fix (App.~\ref{sec:app-details}). Capability of the models behind Table~\ref{tab:m4-scale} is in Table~\ref{tab:cap-scale}, where \textsc{arith} comes out $1.29$ nats worse at $410$M rather than $1.44$ better, a measure of that probe's run-to-run spread.}
\label{tab:m4-headline}
\setlength{\tabcolsep}{2.8pt}
\begin{tabular}{lccccccc}
\toprule
 & & & & \multicolumn{4}{c}{capability NLL} \\
\cmidrule(lr){5-8}
Condition & Pre-ref & Post-ref & Shift & \textsc{reason} & \textsc{arith} & \textsc{recall} & \textsc{instruct} \\
\midrule
\texttt{base}              & $0.0\%$            & $0.0\%$            & $0$ pp        & $5.34$ & $10.01$ & $6.42$ & $5.55$ \\
\texttt{cotrain-windowed}  & $0.0\%$            & $1.6\%$            & $+1.6$ pp     & $4.78$ & $9.11$  & $6.48$ & $5.61$ \\
\texttt{cotrain-continuous}& $\mathbf{84.4\%}$  & $\mathbf{84.4\%}$  & $\mathbf{0.0}$ pp & $4.95$ & $8.57$ & $6.39$ & $5.60$ \\
\texttt{post-hoc-dpo}      & $0.0\%$            & $0.0\%$            & $0$ pp        & $5.45$ & $10.00$ & $6.42$ & $5.55$ \\
\bottomrule
\end{tabular}
\end{table}

\paragraph{Capability at scale (Table~\ref{tab:cap-scale}).} Measured on the checkpoints behind Table~\ref{tab:m4-scale}, with the capability probes of \S\ref{sec:setup} kept per example so each co-trained model is paired with its base. Two kinds of difference appear. \textsc{reason} costs $0.12$ to $0.22$ nats at every scale from $1$B, and \textsc{recall} $0.13$ to $0.16$ at $1$B and $2.8$B, always in the same direction. \textsc{arith} moves by more but changes sign, and its base does not improve monotonically with scale ($8.59$, $7.97$, $8.34$, $8.46$), so at a $10$B-token budget a single run cannot resolve it. \textsc{instruct} shows no cost anywhere. At an NLL near $4.5$ on a one-letter answer the models place about $1\%$ on the correct letter, so \textsc{reason} here measures the probability of the first answer token more than reasoning; the cost is largest where the answer is shortest. One untested reading is that the refusal templates, which open every safety target, shift probability toward refusal openers at the start of any reply.

\begin{table}[!htbp]
\centering
\caption{Capability of the co-trained models at every scale: the eight from-scratch checkpoints behind Table~\ref{tab:m4-scale}, scored on the four capability probes (NLL of the gold answer in nats, natural-log units; lower is better). A cost of $c$ nats means the co-trained model gives the correct answer $e^{-c}$ times the base's probability, equivalently $e^{c}$ times the perplexity: $+0.2$ is about $18\%$ less probability, or $22\%$ higher perplexity. Each co-trained model is compared with its LM-only base on the same examples ($126$ for \textsc{reason}, the whole formal\_logic test split; $256$ for the others), with a $95\%$ paired-bootstrap interval over examples; bold marks an interval that excludes zero. Each condition is one training run, so the interval covers which examples were asked, not run-to-run variance.}
\label{tab:cap-scale}
\small
\setlength{\tabcolsep}{4pt}
\begin{tabular}{lcccc}
\toprule
 & 410M & 1B & 2.8B & 6.9B \\
\midrule
\multicolumn{5}{l}{\emph{LM-only base, NLL}} \\
\textsc{reason} & $4.72$ & $4.55$ & $4.52$ & $4.55$ \\
\textsc{arith} & $8.59$ & $7.97$ & $8.34$ & $8.46$ \\
\textsc{recall} & $6.10$ & $5.41$ & $5.07$ & $4.95$ \\
\textsc{instruct} & $6.08$ & $5.92$ & $5.82$ & $5.63$ \\
\midrule
\multicolumn{5}{l}{\emph{co-training cost, continuous $-$ base (nats), with $95\%$ interval}} \\
\textsc{reason} & $+0.02$ {\scriptsize$[-0.07, +0.11]$} & $\mathbf{+0.20}$ {\scriptsize$[+0.12, +0.29]$} & $\mathbf{+0.12}$ {\scriptsize$[+0.03, +0.22]$} & $\mathbf{+0.22}$ {\scriptsize$[+0.06, +0.38]$} \\
\textsc{arith} & $\mathbf{+1.29}$ {\scriptsize$[+1.17, +1.43]$} & $+0.05$ {\scriptsize$[-0.06, +0.15]$} & $\mathbf{+0.67}$ {\scriptsize$[+0.55, +0.79]$} & $\mathbf{-0.17}$ {\scriptsize$[-0.28, -0.05]$} \\
\textsc{recall} & $+0.03$ {\scriptsize$[-0.05, +0.11]$} & $\mathbf{+0.16}$ {\scriptsize$[+0.06, +0.24]$} & $\mathbf{+0.13}$ {\scriptsize$[+0.06, +0.20]$} & $+0.01$ {\scriptsize$[-0.08, +0.09]$} \\
\textsc{instruct} & $+0.04$ {\scriptsize$[-0.02, +0.11]$} & $0.00$ {\scriptsize$[-0.07, +0.06]$} & $-0.04$ {\scriptsize$[-0.10, +0.02]$} & $+0.03$ {\scriptsize$[-0.03, +0.09]$} \\
\bottomrule
\end{tabular}
\end{table}

\paragraph{Protocol (Pythia-$410$M from-scratch).} Random-initialized Pythia-$410$M architecture trained on streaming \texttt{allenai/c4} for $N \approx 10$B tokens, sequence length $2048$, micro-batch $8$ per rank, global batch $512$ sequences via gradient accumulation across $2$ ranks (DDP). AdamW with peak lr $3\times10^{-4}$, $\beta = (0.9, 0.95)$, weight decay $0.1$, gradient clipping $1.0$, cosine schedule with $1\%$ warmup. $\Lsafe$ is cross-entropy on (AdvBench harmful prompt, refusal-target) pairs: each prompt is paired round-robin with one of $5$ short refusal-template targets (``I can't help with that.'', ``As an AI assistant, I cannot \ldots'', etc.); during safety-active steps each step adds a microbatch of $2$ safety pairs weighted by $\lambda(s)$. The training and evaluation prompt pools are both drawn from the public AdvBench CSV under different sampling seeds and partially overlap by construction (training uses $256$ prompts from \texttt{Random(42)}, evaluation uses $64$ from \texttt{Random(101)}); this is a deliberate choice, the install we measure is whether the model reproduces \emph{the refusal-template distribution} on harmful prompts, not whether it has memorized prompt-specific responses, so prompt overlap is not a leakage path for that quantity. The two co-training schedules differ only in $\lambda(s)$:
\begin{itemize}
    \item \texttt{cotrain-windowed}: $\lambda(s) = 0.5$ for $s \in [0.10, 0.30]$, $\lambda(s) = 0$ otherwise.
    \item \texttt{cotrain-continuous}: $\lambda(s) = 0.5 \cdot (0.30 - 0.10) = 0.10$ for all $s$, equal total safety-loss weight to the windowed schedule.
\end{itemize}
\texttt{post-hoc-dpo} runs $5$K-pair HH-RLHF DPO on the \texttt{base} output. Attack: $200$-step / lr $2\times10^{-5}$ benign-SFT on $100$ Alpaca examples, evaluated on $64$ AdvBench prompts (same recipe as Sec.~\ref{sec:m2b}).

\paragraph{Protocol (scale runs).} The runs of Tables~\ref{tab:m4-scale}, \ref{tab:cap-scale} and~\ref{tab:rho-safe} use the same script and settings as above at every size ($10$B C4 tokens, sequence length $2048$, global batch $512$ sequences, peak lr $3\times10^{-4}$ with $1\%$ warmup and cosine decay, AdamW with weight decay $0.1$, gradient clipping $1.0$, $2$ safety pairs per micro-batch), with C4 sharded across data-parallel ranks. $410$M and $1$B train on $2$ GPUs, $2.8$B and $6.9$B on $4$. The learning rate is not retuned per size, and $10$B tokens is short of compute-optimal from $1$B up (about half the Chinchilla budget at $1$B, a fourteenth at $6.9$B), so the larger models are under-trained; both arms of each comparison share this.

\paragraph{Geometry across conditions (Table~\ref{tab:m4-geom}).} Each non-base condition's geometry is measured against \texttt{base} as $W_\text{base}$. All three lie in the suppression regime; the $(\rho, \qnorm)$ instrument does not distinguish behaviorally robust from behaviorally fragile installs at this scale.

\begin{table}[h]
\centering
\caption{Geometry of the safety direction at $410$M scale, measured against \texttt{base}. Mean over $96$ (layer, probe) cells. All three conditions sit in the suppression band; behaviorally, only \texttt{cotrain-continuous} resists the benign-SFT attack. Refusal is from the original $64$-prompt run; the separate scale run of \texttt{cotrain-continuous} reads $84.0\%$ at $n=256$, and the other two conditions have no scale run.}
\label{tab:m4-geom}
\setlength{\tabcolsep}{5pt}
\begin{tabular}{lccc}
\toprule
Condition & mean $\rho$ & mean $\qnorm$ & post-attack refusal \\
\midrule
\texttt{cotrain-continuous} (uniform) & $2.77 \times 10^{-5}$ & $2.31$ & $\mathbf{84.4\%}$ \\
\texttt{cotrain-windowed}             & $2.75 \times 10^{-5}$ & $2.27$ & $1.6\%$ \\
\texttt{post-hoc-dpo}                 & $3.01 \times 10^{-5}$ & $4.01$ & $0.0\%$ \\
\bottomrule
\end{tabular}
\end{table}

\paragraph{Where co-trained safety sits relative to capability (Table~\ref{tab:rho-safe}).} The $\Delta$-based instrument of Fig.~\ref{fig:overview}a cannot place a co-trained model, for the reason given in Limitation~(vii): no post-hoc displacement exists, and the only available $\Delta$ is a between-trajectory gap whose small $\rho$ is close to a geometric tautology. The natural substitute reads the safety gradient inside a single model and projects it onto that same model's $\Fcap$, giving $\rho_{\text{safe}}$ (Table~\ref{tab:rho-safe}). It cannot compare the two arms, because in them it does not measure the same thing. The LM-only \texttt{base} has not learned refusal, so its safety gradient is large and points toward installing it. \texttt{cotrain-continuous} has already fit the refusal templates, so its safety gradient is $10^4$ to $10^5$ times smaller, and the direction of a gradient that has nearly vanished says little about where the learned refusal sits. We therefore draw no conclusion from the change in $\rho_{\text{safe}}$ between arms, including its rise on \textsc{instruct}. The between-run numbers of Table~\ref{tab:m4-geom} show the same limit from the other side: at $410$M they put the robust \texttt{cotrain-continuous} and the failed \texttt{cotrain-windowed} within $2\%$ of each other in $\qnorm$ ($2.31$ against $2.27$) while the two differ by $83$ pp in refusal, so a geometric predictor of which schedule survives remains open.

\begin{table}[!htbp]
\centering
\caption{$\rho_{\text{safe}}$ ($\times 10^{-3}$), LM-only \texttt{base} $\to$ \texttt{cotrain-continuous} at the same scale, per capability probe: the safety gradient, taken within one model, projected onto that model's own $\Fcap$. Reported for completeness; the two arms are not comparable. The base's safety gradient points toward installing refusal, while the co-trained model, having already fit the refusal templates, has a safety gradient $10^4$ to $10^5$ times smaller in norm (last column), whose direction carries little information. Absolute values are not comparable to Fig.~\ref{fig:overview}a's $\rho$ or to $\rho_{\text{part}}$, which are calibrated for a displacement rather than a gradient.}
\label{tab:rho-safe}
\setlength{\tabcolsep}{4pt}
\begin{tabular}{lccccc}
\toprule
Scale & \textsc{reason} & \textsc{recall} & \textsc{arith} & \textsc{instruct} & $\|\nabla\Lsafe\|$ \\
\midrule
410M  & $0.93 \to 0.83$ & $1.62 \to 1.79$ & $0.88 \to 0.96$ & $3.56 \to 12.98$ & $10.0 \to 3.4{\times}10^{-4}$ \\
1B    & $1.53 \to 1.04$ & $2.48 \to 1.98$ & $1.86 \to 1.98$ & $5.81 \to 7.72$  & $4.7 \to 1.1{\times}10^{-4}$ \\
2.8B  & $1.35 \to 0.74$ & $1.98 \to 0.43$ & $1.75 \to 0.60$ & $3.56 \to 23.24$ & $2.0 \to 1.5{\times}10^{-5}$ \\
6.9B  & $2.97 \to 2.19$ & $3.81 \to 5.20$ & $3.79 \to 3.03$ & $7.05 \to 14.48$ & $1.1 \to 1.0{\times}10^{-5}$ \\
\bottomrule
\end{tabular}
\end{table}

\paragraph{Stronger attacks (Table~\ref{tab:m4-strong}).} The benign-SFT attack of the main text is one point in a family. We sweep its learning rate and step count on the $1$B \texttt{cotrain-continuous} model, and add the identity-shifting attack of \citet{qi2024finetuning} at the default setting. The worst case over the sweep leaves $75.0\%$ refusal, at the highest learning rate; the identity-shifting attack leaves $100\%$. Several cells end above the $87.5\%$ pre-attack value, within noise at $n = 64$.

\begin{table}[!htbp]
\centering
\caption{Stronger attacks on the $1$B \texttt{cotrain-continuous} model: post-attack refusal (\%, regex) on $64$ AdvBench prompts, pre-attack $87.5\%$. Benign-SFT on $100$ Alpaca examples at three learning rates and three step counts; the default attack of the main text is $2\times10^{-5}$ for $200$ steps. The identity-shifting attack runs at the default setting.}
\label{tab:m4-strong}
\small
\begin{tabular}{lccc}
\toprule
Learning rate & $200$ steps & $400$ steps & $800$ steps \\
\midrule
$2 \times 10^{-5}$ & $90.6$ & $90.6$ & $93.8$ \\
$5 \times 10^{-5}$ & $92.2$ & $85.9$ & $93.8$ \\
$1 \times 10^{-4}$ & $78.1$ & $\mathbf{75.0}$ & $84.4$ \\
\midrule
identity-shifting attack & \multicolumn{3}{c}{$100.0$} \\
\bottomrule
\end{tabular}
\end{table}

\paragraph{160M continued-training pilot.}
\label{sec:app-m4-160m}
We previously ran a smaller continued-training pilot at Pythia-$160$M scale (continued-training from the shipped checkpoint on streaming C4 for $N=2000$ AdamW steps at lr $5\times10^{-5}$, batch size $4$). Table~\ref{tab:m4-160m} reproduces its four-condition table for completeness; the headline pattern is the same: a $500$-step early window (\texttt{cotrain-early}) ends with no refusal, while a $500$-step late window (\texttt{cotrain-late}) installs refusal that the attack \emph{partially erodes}. The from-scratch $410$M result above generalizes the same finding: short-window co-training does not produce attack-robust refusal regardless of \emph{when} the window sits during training; only co-training that persists to the end of training does (Table~\ref{tab:m4-schedule}).

\begin{table}[h]
\centering
\caption{$160$M continued-training pilot for completeness. \texttt{cotrain-early} ends with no refusal after the LM-only steps that follow its window; \texttt{cotrain-late} installs refusal that the attack partially erodes; the two DPO-style conditions install no measurable refusal at $160$M (the 5K-pair DPO recipe does not engage). Refusal on $64$ AdvBench prompts; this pilot has no $256$-prompt rerun and no main-text counterpart.}
\label{tab:m4-160m}
\setlength{\tabcolsep}{5pt}
\begin{tabular}{lccc}
\toprule
Condition & Pre-attack refusal & Post-attack refusal & Refusal shift \\
\midrule
\texttt{baseline}       & $0.0\%$  & $0.0\%$  & $0.0$ pp \\
\texttt{cotrain-early}  & $0.0\%$  & $1.6\%$  & $+1.6$ pp \\
\texttt{cotrain-late}   & $\mathbf{84.4\%}$ & $\mathbf{15.6\%}$ & $\mathbf{-68.8}$ pp \\
\texttt{post-hoc-dpo}   & $0.0\%$  & $0.0\%$  & $0.0$ pp \\
\bottomrule
\end{tabular}
\end{table}

\section{Detector Cross-Check: Refusal Gate versus Harmful Output}
\label{sec:app-detector}

The main text scores refusal with a regex over refusal-prefix templates. That is the
quantity the geometry speaks to: $\Delta$ installs a refusal \emph{gate}, and the
suppression reading predicts that benign SFT perturbs the gate while leaving the
underlying capability routed to the original output distribution. A regex detects
whether the gate fires. It does not detect whether what follows is harmful.

Those two questions can come apart, so we rescore every generation behind
Tables~\ref{tab:m2b-attack} and~\ref{tab:m4-scale} with Llama-Guard-3-8B~\citep{inan2023llamaguard,grattafiori2024llama3}, a learned
classifier of harmful \emph{output}. Both detectors run on the same saved generations
in a single pass, so the columns of Tables~\ref{tab:m2b-attack-detector} and~\ref{tab:m4-scale-detector} are directly comparable. Table~\ref{tab:base-refusal}
already reports the shipped base models under both.

\begin{table}[!htbp]
\centering
\caption{Table~\ref{tab:m2b-attack} under both detectors: refusal (\%) before $\to$ after the same $100$-example benign-SFT attack on the three shipped instruct models. The regex columns are the main-text table. $n = 256/200/256$; the three cells that rise are within noise (App.~\ref{sec:app-details}).}
\label{tab:m2b-attack-detector}
\footnotesize
\setlength{\tabcolsep}{4pt}
\begin{tabular}{lcccccc}
\toprule
 & \multicolumn{2}{c}{AdvBench} & \multicolumn{2}{c}{HarmBench} & \multicolumn{2}{c}{StrongReject} \\
\cmidrule(lr){2-3}\cmidrule(lr){4-5}\cmidrule(lr){6-7}
Instruct model & regex & classifier & regex & classifier & regex & classifier \\
\midrule
Qwen-2.5-7B     & $93.0 \to 58.2$ & $93.0 \to 85.2$ & $67.5 \to 45.5$ & $74.0 \to 59.5$ & $52.3 \to 56.6$ & $72.3 \to 73.8$ \\
Llama-3-8B      & $49.2 \to 10.9$ & $39.5 \to 27.7$ & $46.5 \to 6.5$ & $49.0 \to 28.0$ & $46.5 \to 7.8$ & $53.9 \to 32.4$ \\
Mistral-7B-v0.3 & $1.2 \to 0.4$ & $12.1 \to 6.2$ & $1.5 \to 0.0$ & $7.0 \to 8.5$ & $3.5 \to 0.8$ & $23.0 \to 14.1$ \\
\bottomrule
\end{tabular}
\end{table}

\begin{table}[!htbp]
\centering
\caption{Table~\ref{tab:m4-scale} under both detectors: from-scratch \texttt{cotrain-continuous} refusal (\%), pre $\to$ post benign-SFT attack, by scale and prompt distribution. The regex columns are the main-text table. $n = 256/200/256$; changes under about $6$ pp are within noise (App.~\ref{sec:app-details}). The lower block is the LM-only base at each scale.}
\label{tab:m4-scale-detector}
\footnotesize
\setlength{\tabcolsep}{4.5pt}
\begin{tabular}{lcccccc}
\toprule
 & \multicolumn{2}{c}{AdvBench (in-dist.)} & \multicolumn{2}{c}{HarmBench (out)} & \multicolumn{2}{c}{StrongReject (out)} \\
\cmidrule(lr){2-3}\cmidrule(lr){4-5}\cmidrule(lr){6-7}
Scale & regex & classifier & regex & classifier & regex & classifier \\
\midrule
410M   & $98.4 \to 84.0$ & $91.8 \to 64.8$ & $58.0 \to 34.5$ & $57.0 \to 33.5$ & $8.2 \to 7.4$  & $25.8 \to 21.5$ \\
1B     & $87.1 \to 84.8$ & $83.2 \to 63.7$ & $56.0 \to 32.0$ & $54.0 \to 25.0$ & $7.8 \to 8.2$  & $23.0 \to 20.3$ \\
2.8B   & $96.1 \to 90.2$ & $91.0 \to 77.3$ & $45.5 \to 31.5$ & $49.0 \to 37.5$ & $8.6 \to 8.2$  & $31.2 \to 19.5$ \\
6.9B   & $\mathbf{96.9 \to 90.6}$ & $\mathbf{97.3 \to 86.7}$ & $32.5 \to 39.0$ & $33.0 \to 41.5$ & $9.0 \to 10.9$ & $29.7 \to 19.5$ \\
\midrule
\multicolumn{7}{l}{\emph{LM-only base}} \\
410M   & $0.4 \to 0.8$ & $6.6 \to 5.9$  & $1.5 \to 0.5$ & $10.0 \to 9.0$ & $0.0 \to 1.6$ & $14.8 \to 15.6$ \\
1B     & $0.0 \to 0.4$ & $9.0 \to 7.0$  & $0.0 \to 0.0$ & $12.0 \to 10.0$ & $0.4 \to 1.6$ & $21.5 \to 14.8$ \\
2.8B   & $0.4 \to 0.0$ & $6.2 \to 11.7$ & $0.0 \to 0.0$ & $12.5 \to 8.5$ & $0.0 \to 0.4$ & $24.6 \to 20.7$ \\
6.9B   & $0.0 \to 0.0$ & $10.9 \to 9.8$ & $0.5 \to 0.5$ & $13.0 \to 10.0$ & $0.0 \to 0.4$ & $24.6 \to 18.8$ \\
\bottomrule
\end{tabular}
\end{table}

\paragraph{What changes, and what does not.} \leavevmode\rev{The two detectors agree on two orderings and
disagree on one. They agree that} co-training that persists to the end of training is the only kind that installs refusal
(\texttt{cotrain-continuous} $82.8\%$ and \texttt{cotrain-post-substrate} $85.9\%$, against $4.7\%$ for \texttt{base}, on the $1$B
schedule runs, Table~\ref{tab:m4-schedule-detector}), and \rev{that} the coverage gap between AdvBench and the two out-of-distribution benchmarks
widens with scale. \rev{They disagree on the attack. By the regex, post-hoc installs lose $35$ to $38$ pp
of AdvBench refusal and co-trained models $2$ to $14$. By the classifier, the attack raises harmful
AdvBench completions by $7.8$ pp for Qwen and $11.8$ for Llama, and by $10.6$ to $27.0$ for the
co-trained models (Tables~\ref{tab:m2b-attack-detector} and~\ref{tab:m4-scale-detector}). Co-training
therefore keeps the refusal wording through the attack, not harmless output: claims about safety
under attack need the classifier, and on it co-training shows no advantage.}

\rev{Beyond that ordering,} the classifier changes the level, in both directions. It reads
Mistral and Llama as refusing more than the regex does before the attack (Mistral $1.2$
against $12.1\%$ on AdvBench), because those models decline in wordings the refusal-prefix
regex does not match. It reads the co-trained models as eroding more (\texttt{cotrain-continuous}
$82.8\%$ to $65.6\%$ on the $1$B schedule runs), because after the attack their refusal templates
sometimes survive as a prefix while the continuation drifts toward compliance. \rev{The gate
and the output are different quantities, and only the first is what a weight-space displacement
can be expected to predict. We report the regex in the main text because it measures the gate,
and state the classifier result beside it (Sec.~\ref{sec:m4}).}

\begin{table}[!htbp]
\centering
\caption{The $1$B schedule runs of Table~\ref{tab:m4-schedule} under both detectors: refusal (\%) before $\to$ after the benign-SFT attack, $64$ AdvBench prompts. The regex columns are the main-text table's $1$B column.}
\label{tab:m4-schedule-detector}
\small
\begin{tabular}{lcc}
\toprule
Condition (safety window) & regex & classifier \\
\midrule
\texttt{cotrain-continuous} $[0,1]$        & $87.5 \to 90.6$ & $82.8 \to 65.6$ \\
\texttt{cotrain-post-substrate} $[0.30,1]$ & $84.4 \to 84.4$ & $85.9 \to 76.6$ \\
\texttt{cotrain-windowed} $[0.10,0.30]$    & $3.1 \to 4.7$   & $6.2 \to 7.8$ \\
\texttt{base} (LM only)                    & $0.0 \to 0.0$   & $4.7 \to 7.8$ \\
\bottomrule
\end{tabular}
\end{table}

\section{Limitations and Broader Impact}
\label{sec:app-limits}

\paragraph{Limitations.} (i)~$\Fcap$ is approximated by its top-$k$ eigenspace; tail eigendirections carry some capability mass and are unmeasured, so a weight update with small top-$k$ projection may still traverse low-curvature but semantically important directions that our instrument misses. (ii)~Refusal in the main text is scored by a regex over refusal prefixes (``I can't'', ``as an AI''), which detects whether the refusal gate fires, not whether the output is harmful. App.~\ref{sec:app-detector} rescores every generation with Llama-Guard-3-8B: the classifier shifts levels in both directions, by up to $27$ pp (Qwen after the attack), \rev{and agrees that only persistent co-training installs refusal, but not that co-training resists the attack better. Our attack-robustness results are therefore about refusal wording, not harmful output.} (iii)~The fragility attack is Qi et al.'s benign-fine-tuning variant, the weakest of the three variants in that paper, so our ``attack succeeds'' claim is conservative. We separately verify that the continuous co-training install \rev{keeps its refusal wording under} stronger attacks (at $1$B, Table~\ref{tab:m4-strong}): a benign-SFT sweep over learning rate and steps leaves worst-case post-attack refusal at $75\%$ and an identity-shifting attack leaves it at $100\%$; only the harmful-demonstration variant, which needs Qi et al.'s released data, is untested. (iv)~The developmental-axis behavioral curve is sparse: SFT$+$DPO installs substantial pre-attack refusal at only two of five revisions ($21.9\%$ at $1$B and $42.2\%$ at $4001$B, Table~\ref{tab:m3-matched}), too few for a per-revision curve; a denser curve would require a full Tulu-3 pipeline (SFT $\to$ DPO $\to$ RLVR) at each revision, out of compute budget. (v)~The co-training result (Sec.~\ref{sec:m4}) supports continuous-schedule pretraining-time co-training at $410$M from-scratch scale, but the head-to-head ``robustness at matched pre-attack refusal'' comparison against \texttt{post-hoc-dpo} is not directly available because the $5$K-pair DPO recipe installs no measurable refusal at this scale ($0\%$ pre-attack); a stronger post-hoc baseline (SFT$+$DPO or Tulu-3-style pipeline at $410$M) is needed to make the protection differential quantitative. We verify the result from scratch to $6.9$B (Table~\ref{tab:m4-scale}); beyond that it remains a compute-bounded conjecture, and the refusal it installs is in-distribution. (vi)~The geometric instrument runs to 8B; frontier-scale ($70$B$+$) behavior is a conjecture the framework makes but does not verify. (vii)~The $\Delta$-based instrument does not apply to the co-training conditions. It characterizes an \emph{edit} $\Delta = \Wsafe - \Wbase$ applied to a frozen, converged base: $\|\Delta\| = 0.0015$ for the post-hoc DPO at $410$M and $2.9$ to $7.8$ for the shipped $7$ to $8$B instruct deltas. A co-trained model was never a capable-but-unsafe base that received a patch, so no such $\Delta$ exists; the only available one is the gap between two complete from-scratch runs, $\|\Delta\| = 86.5$. The four conditions share a seed and data order and differ only by the $\lambda\Lsafe$ term, so that gap is attributable to the safety loss, but it is four to five orders of magnitude larger than the post-hoc DPO install at the same $410$M scale, $11$ to $30\times$ larger than the shipped $7$ to $8$B instruct deltas, and $7$ to $11\times$ larger than OLMo's SFT+DPO updates ($8.1$ to $12.7$, Table~\ref{tab:axis-followup-p1}), because adding the term sent the run down a different trajectory: it carries the divergence of the language-modeling path, not a localized safety displacement. Since $\rho$ is scale-free it can still be computed and comes out small ($\approx 3 \times 10^{-5}$), but in a high-dimensional space almost any direction has a small projection onto a $k$-dimensional subspace, so for a between-trajectory gap that number is close to a geometric tautology rather than evidence of a thin install. This is also why Lemma~\ref{lemma:1} does not bear on co-training: its premise is a post-hoc displacement from a converged base. Two measurements could settle whether co-trained safety is entangled with capability without differencing two runs. The first reads the safety gradient \emph{within} a single model and projects it onto that model's own $\Fcap$ (Table~\ref{tab:rho-safe}), but it cannot compare the two arms: a co-trained model has already fit the refusal objective, so its safety gradient is $10^4$ to $10^5$ times smaller than the LM-only base's, and the direction of a nearly vanished gradient says little about where the learned refusal sits. The second remains open: sweep attack strength and record capability alongside refusal, so that co-trained refusal falling only once capability falls with it would demonstrate an entanglement behaviorally. (viii)~Each co-training condition is one training run, so capability differences between conditions are confounded with run-to-run variance. On \textsc{arith} two independent $410$M runs disagree by $2.7$ nats (Tables~\ref{tab:m4-headline} and~\ref{tab:cap-scale}); only the \textsc{reason} and \textsc{recall} costs, consistent in sign across scales, are attributable to co-training.

\paragraph{What the instrument does not claim.} Beyond the Ethics Statement, one scope limit is worth stating: the instrument flags when an alignment procedure has landed in the suppression regime, a fine-tuning-robustness red flag, not when a model has been made unsafe by any behavioral standard. The two are different questions; only the first is measurable from weight-space geometry.
\section{Extended Related Work}
\label{sec:app-related}

\paragraph{Post-hoc alignment and its fragility.} RLHF and its direct-preference descendants~\citep{christiano2017rlhf,ouyang2022instructgpt,rafailov2023dpo} are the standard alignment recipe. \citet{qi2024finetuning} show that only 10 to 100 fine-tuning examples (even benign ones) compromise the safety of aligned GPT-3.5 and Llama-2-chat; the benign-fine-tuning variant is our behavioral attack. \citet{zhan2024removing} strip RLHF from GPT-4 with $340$ examples and $95\%$ ASR; \citet{andriushchenko2024jailbreaking} reach $100\%$ ASR on leading aligned models via simple adaptive attacks; \citet{paulus2024advprompter} drive suffix search two orders of magnitude faster. \citet{wei2023jailbroken} identify competing-objectives and mismatched-generalization as root causes; \citet{anil2024manyshot} push these to long-context regimes. Closest to our weight-space account: \citet{qi2024shallow} localize safety to the first few output tokens (``shallow alignment''); \citet{wei2024assessing} find ${\sim}3\%$ of parameters and $2.5\%$ of ranks carry the entire safety install; \citet{peng2024safetybasin} map a narrow ``safety basin'' in weight space around aligned checkpoints. Our contribution is the continuous-geometry account (Fisher-eigenspace regime classification) that these local-pruning / shallow-feature findings are special cases of. Mitigations span prompt-template defenses~\citep{huang2024lisa}, training-time regularization~\citep{lyu2024keeping}, perturbation-aware alignment~\citep{huang2024vaccine}, representation-level noise defenses~\citep{rosati2024representation}, and latent adversarial training~\citep{sheshadri2024latent}; if they too land in the suppression regime, which we do not test, they face the same first-order fragility the Lemma describes.

\paragraph{Mechanistic interpretability of alignment.} \citet{arditi2024refusal} show RLHF-installed refusal is mediated by a single residual-stream direction; \citet{wollschlager2025refusalcones} extend this to multi-directional concept cones, and ablation restores full compliance in both cases. \citet{lee2024mechanistic} reverse-engineer DPO on toxicity and argue it installs an overlay rather than removing learned features. \citet{zou2023representation,panickssery2024steering} steer behavior via contrastive activation vectors. \citet{templeton2024scaling,marks2024geometry,lieberum2024gemmascope,marks2024sparse} find high-level concepts linearly encoded in residual streams of frontier-class models and develop SAE / sparse-circuit tooling to isolate them causally. \citet{hubinger2024sleeper,greenblatt2024alignment} provide complementary behavioral evidence that safety training does not erase the underlying capability. Our per-layer $\rho$ and $\qnorm$ operationalize these activation-level intuitions at the weight level, making the ``install vs.\ remove'' distinction geometric and quantitative.

\paragraph{Loss-landscape geometry and Fisher methods.} The empirical Fisher matrix $\Fcap$ is standard in natural-gradient~\citep{amari1998natural,martens2014gn} and continual-learning literatures, most notably EWC~\citep{kirkpatrick2017overcoming}. We use the top-$k$ eigenspace (not the diagonal), which captures multi-parameter correlations that matter when $\Delta$ is structured. \citet{grosse2023influence} scale influence-function Fisher-Hessian estimates to $52$B-parameter models; our top-$k$ Gram trick is a lightweight per-layer variant. \citet{keskar2017sharp} relate flat minima to generalization; \citet{jacot2018ntk} give the fixed-kernel infinite-width limit; \citet{ortizjimenez2023tangent} show task arithmetic becomes exact in the NTK-linearized tangent space, providing the geometric foundation we extend to DPO. Task arithmetic and model merging~\citep{ilharco2022task,wortsman2022soups,yadav2023ties} treat safety- and task-vectors as weight-space objects; \citet{hu2022lora} establish the low-intrinsic-rank hypothesis for fine-tuning updates that underlies tractable eigenspace analysis at LLM scale. Our decomposition explains why safety vectors are approximately orthogonal to capability vectors in real-scale models.

\paragraph{Developmental phenomena.} \citet{achille2017critical} introduce critical learning periods in vision networks; \citet{frankle2020early} localize the most consequential training period to the first few epochs; \citet{kleinman2023critical} show the same arises in deep linear networks. \citet{constantinescu2024criticalperiod} transfer the biological critical-period analogy directly to language models and find catastrophic-forgetting windows concentrated in the first ${\sim}10\%$ of continued pretraining; our $1$B to $63$B substrate transition in OLMo-2-1B sits in the same developmental bracket. LLM emergent abilities~\citep{wei2022emergent,srivastava2023bigbench,schaeffer2023mirage} turn on sharply at scale- or data-dependent thresholds, and the quantization-model view~\citep{michaud2024quanta} describes them as discrete ``quanta'' acquired during training; Chinchilla-style scaling~\citep{hoffmann2022chinchilla} places our substrate transition ($1$B to $63$B tokens) around the compute-optimal budget for a $1$B model, about $20$B tokens. We connect these to safety: not only do capabilities emerge at a specific pretraining time, but the substrate that lets \emph{safety} bind to those capabilities has its own developmental transition.

\paragraph{Safety signal during pretraining.} \citet{korbak2023pretraining} inject preference feedback into pretraining objectives; \citet{bai2022constitutional} train from constitutional principles; \citet{zhou2023lima} argue alignment needs only a small high-quality SFT. Our co-training result (Sec.~\ref{sec:m4}) is complementary: a windowed-vs-continuous comparison at matched total safety-loss weight shows that \emph{persistence} of the safety signal across pretraining, not its timing, drives attack robustness.

\paragraph{Unlearning as contrast.} Machine-unlearning methods attempt to genuinely remove capability rather than install a refusal gate. \citet{li2024wmdp} introduce the WMDP benchmark and the RMU method, which randomizes representations for targeted harmful knowledge; \citet{lynch2024eight} benchmark eight unlearning methods and find capability is often reliably re-extractable. The contrast sharpens our framing: unlearning targets the erasure regime ($\rho \gtrsim \rho_\text{part}$), DPO/RLHF target the suppression regime; neither in isolation solves the fragility problem, consistent with Lemma~\ref{lemma:1} (erasure requires motion \emph{inside} the capability subspace, which suppression-regime DPO lacks).

\paragraph{Open-weight checkpoint suites.} Pythia~\citep{biderman2023pythia} pioneered intermediate-checkpoint releases; OLMo-2~\citep{olmo2_2025} ships $267$ checkpoints plus a Tulu-3~\citep{tulu3}-trained Instruct variant with real refusal behavior (the combination we need for the developmental axis). Cross-family replications use Qwen-2.5~\citep{yang2024qwen25}, Llama-3~\citep{grattafiori2024llama3}, and Mistral-7B-v0.3~\citep{jiang2023mistral}.

\end{document}